\pdfoutput=1
\documentclass[11pt,letterpaper]{article}

\usepackage[margin=1in]{geometry}
\usepackage[utf8]{inputenc}
\usepackage[T1]{fontenc}
\usepackage{lmodern}
\usepackage{microtype}


\usepackage{graphicx}
\usepackage{xcolor}
\usepackage{booktabs}

\usepackage{amsmath,amssymb,amsthm,mathtools}

\usepackage[round,authoryear]{natbib}
\setcitestyle{authoryear,round,citesep={;},aysep={,},yysep={;}}
\definecolor{linkblue}{rgb}{0.05,0.20,0.55}
\usepackage[colorlinks=true,linkcolor=linkblue,citecolor=linkblue,urlcolor=linkblue,
            pdftitle={First-Order Stationarity of Reverse Diffusions}]{hyperref}

\numberwithin{equation}{section}

\newtheorem{theorem}{Theorem}[section]
\newtheorem{proposition}[theorem]{Proposition}
\newtheorem{lemma}[theorem]{Lemma}
\newtheorem{corollary}[theorem]{Corollary}
\theoremstyle{definition}
\newtheorem{assumption}[theorem]{Assumption}

\theoremstyle{remark}
\newtheorem{remark}[theorem]{Remark}

\newcommand{\R}{\mathbb R}
\newcommand{\E}{\mathbb E}
\newcommand{\dd}{\,\mathrm d}
\newcommand{\KL}{\operatorname{KL}}
\newcommand{\op}{\mathrm{op}}
\newcommand{\Id}{I_d}
\newcommand{\Itwo}{I_{2d}}
\newcommand{\F}{\mathcal F}
\newcommand{\Fv}{\mathcal F_v}

\newcommand{\ip}[2]{\left\langle #1,#2\right\rangle}
\newcommand{\norm}[1]{\left\lVert #1\right\rVert}
\newcommand{\eps}{\varepsilon}
\newcommand{\hq}{\hat q}
\newcommand{\hQ}{\hat Q}
\newcommand{\hp}{\hat p}

\newcommand{\hY}{\hat Y}
\newcommand{\kernel}{\mathsf K}

\title{First-Order Stationarity of Reverse Diffusions}

\author{%
  Zhifeng Chen \thanks{Email: \texttt{zchen925@wisc.edu}}\qquad
   Chenyang Jiang \qquad
  Yazhen Wang \\[0.6ex]
  {Department of Statistics, UW-Madison}%
}
\date{}   % arXiv stamps the submission date itself; use \date{\today} to print one

\begin{document}

\maketitle

\begin{abstract}
Recent literature has shown a strong connection between optimization and sampling. We develop the corresponding first-order theory for diffusion models. First, the SDE-based reverse-time flows of overdamped and underdamped Langevin diffusions contract relative Fisher divergences at explicit exponential rates whenever the stationary potential of the forward process is strongly convex---a condition on the noising process one chooses, not on the data. This is a unique advantage of SDE-based reverse diffusion, absent in the reverse process based on ODEs. Second, we incorporate discretization and establish averaged first-order stationarity bounds---the sampling analog of averaged gradient-norm guarantees in nonconvex optimization---for samplers of both overdamped and underdamped diffusion models. As in nonconvex optimization, the convexity-free certificate is local: it guarantees score consistency, not global mode weights.
\end{abstract}

\section{Introduction}
\label{sec:intro}

Diffusion models generate data in two stages: a forward Markov process gradually corrupts samples from an unknown data distribution $\rho_0$ until their law is essentially a tractable stationary distribution, and generation simulates this process backward in time, guided by the score functions $\nabla\log\rho_s$ of the forward marginals, learned by (denoising) score matching \citep{hyvarinen2005estimation,vincent2011connection,sohl2015deep,song2019generative,ho2020denoising,song2021score}. Three approximations separate the implemented sampler from the ideal reverse process---initialization from the reference law rather than the true terminal marginal, time discretization, and score estimation---and a large body of work bounds their combined effect on the \emph{terminal} law, in total variation, KL, or Wasserstein distance.

This paper takes a different, more structural view. Rather than asking how far the final sample is from the data, we ask what the reverse dynamics itself---exact or discretized---does to a \emph{first-order} measure of mismatch. The central objects are the \emph{relative score} and the \emph{relative Fisher divergence}: for positive densities $p$ and $q$,
\[
    \F(p\,\|\,q)=\E_{p}\Big\|\nabla\log\frac pq\Big\|^2 .
\]
The relative score vanishes identically if and only if $p=q$, so $\F(p\,\|\,q)$ is a natural stationarity measure: it plays the role of the squared gradient norm $\|\nabla f\|^2$, as $\KL(\cdot \,\|\,q)$ plays that of the objective value $f$. The analogy is not merely formal. The Fokker--Planck equation of overdamped Langevin dynamics is the gradient flow of relative entropy in the Wasserstein geometry \citep{jordan1998variational}, and the relative Fisher divergence is exactly its gradient norm. We will show in this work that this is also true for the \emph{reverse} flow:
$\frac{\dd}{\dd t}\KL(p_t\,\|\,q_t)=-\F(p_t\,\|\,q_t)$
(Corollary~\ref{cor:kl-dissipation})---the exact analog of the descent identity $\frac{\dd}{\dd t}f(x_t)=-\|\nabla f(x_t)\|^2$ along a gradient flow. Table~\ref{tab:dictionary} records the full dictionary the paper centers around.

Most rows of this dictionary are not new: they underlie the first-order theory of forward Langevin Monte Carlo \citep{balasubramanian2022towards} and Fisher-information dissipation along forward flows \citep{wibisono2025mixing}. The contribution of this paper is that the dictionary survives time reversal, where the reference law $q_t$ is no longer a fixed stationary distribution but a moving target set by the noising schedule. We show that this symmetry holds for both overdamped and underdamped reverse-time SDEs, but does not extend to their probability flow ODE counterparts.

\begin{table}[t]
\centering
\small
\begin{tabular}{ll}
\toprule
Smooth optimization & Reverse diffusion sampling\\
\midrule
objective $f(x)$ & relative entropy $\KL(p\,\|\,q)$\\
gradient $\nabla f$ & relative score $\nabla\log(p/q)$ in $L^2(p)$\\
squared gradient norm $\|\nabla f\|^2$ & Fisher divergence $\F(p\,\|\,q)$\\
learning rate $\eta$ & step size $h$ \\
descent identity $\frac{\dd}{\dd t}f=-\|\nabla f\|^2$ & KL dissipation (Corollary~\ref{cor:kl-dissipation})\\
strong convexity of $f$ & strong log-concavity of the forward potential $U$\\
$\|\nabla f(x_t)\|^2\le e^{-2mt}\|\nabla f(x_0)\|^2$ & exponential Fisher contraction (Theorems~\ref{thm:overdamped-decay}, \ref{thm:underdamped-decay})\\
Polyak--{\L}ojasiewicz inequality & log-Sobolev inequality for $q$\\
$\frac1N\sum_k\|\nabla f(x_k)\|^2\lesssim\frac{f(x_0)-f^\star}{\eta N}+O(\eta)$ & averaged Fisher bounds (Theorems~\ref{thm:overdamped-stationarity}, \ref{thm:underdamped-stationarity})\\
change of variables $x\mapsto\Phi(x)$  & probability-flow ODE (Proposition~\ref{prop:ode-transport})\\
\bottomrule
\end{tabular}
\caption{The optimization dictionary used throughout the paper.}
\label{tab:dictionary}
\end{table}

\paragraph{Why first-order guarantees?}
In smooth nonconvex optimization one cannot certify convergence to a global minimum, yet gradient descent still admits the averaged guarantee $\frac1N\sum_k\|\nabla f(x_k)\|^2\lesssim\frac{f(x_0)-f^\star}{\eta N}+O(\eta)$ \citep{ghadimi2013stochastic}. The sampling situation is parallel: for non-log-concave data, converting score accuracy into KL or total-variation convergence requires functional inequalities, whereas the relative Fisher divergence remains meaningful \emph{without such assumptions}---driving it to zero certifies
that the sampler's law is locally consistent with its target, the two scores agreeing on average in $L^2$. The analogy also makes the limitation transparent: as a stationary point need not be a minimum, a law can have small relative Fisher divergence while misallocating mass between well-separated modes. This limitation parallels the gap between stationarity and global optimality in nonconvex optimization.

\subsection{Related work}

\paragraph{Convergence theory of diffusion models.}
Time reversal of diffusions goes back to \citet{anderson1982reverse} and
\citet{haussmann1986time}; diffusion models build directly on this theory,
including the critically damped kinetic variant of \citet{dockhorn2022score}.
A growing line of quantitative convergence theory controls initialization,
discretization, and score-estimation errors simultaneously for essentially
arbitrary data distributions \citep{lee2023convergence, chen2023sampling, jiao2025optimal}.

\paragraph{Connections between sampling and optimization.}
A complementary line of work draws connections between sampling dynamics and
optimization. The classical variational formulation of the Fokker--Planck
equation interprets overdamped Langevin dynamics as a Wasserstein gradient flow
of the relative entropy \citep{jordan1998variational}, motivating the viewpoint
of sampling as optimization over the space of probability measures
\citep{wibisono2018sampling}. Momentum-based sampling admits an even closer
analogy with accelerated optimization: \citet{ma2021nesterov} interpreted
underdamped Langevin dynamics as an analog of Nesterov acceleration for
minimizing the KL divergence, with hypocoercivity playing a role analogous to
the coupling between position and momentum in accelerated optimization.

\paragraph{Fisher-divergence contraction.}
Recent work has also studied the evolution of Fisher divergence along
isotropic Fokker--Planck channels. \citet{wibisono2025mixing} established contraction of Fisher
divergence under suitable convexity assumptions. Closest in spirit to our work,
\citet{balasubramanian2022towards} developed Fisher-information guarantees for forward
Langevin Monte Carlo, viewing Fisher divergence as a sampling analog of
first-order stationarity in optimization. Our work develops a corresponding
perspective for \emph{reverse-time} diffusion, where contraction must be established against a reference that itself evolves.

\subsection{Contributions and paper organization}
\paragraph{Contraction in continuous reverse flow.}
Our results form two groups, one about exact reverse flows and one about discrete samplers.
If the stationary potential $U$ of the forward diffusion is strongly convex, then the reverse SDE contracts score mismatch exponentially, for both overdamped and underdamped diffusion. The assumptions concern only the forward potential, and no log-Sobolev inequality is imposed on the data: the continuous reverse dynamics forgets initialization error exponentially fast for almost arbitrary data.

\paragraph{Averaged Fisher divergence bound for discretized algorithms.}
(i) For the frozen-score exponential integrator of the reverse overdamped process---the idealized DDPM-type update---with spatially Lipschitz scores and sub-Gaussian data assumptions, the numerical law $\hp_t$ satisfies $\frac1T\int_0^T\F(\hp_t\,\|\,q_t)\dd t\lesssim H_0/T+L_s^2dh+O(h^2)$ against the exact reverse marginals $q_t$ (Theorem~\ref{thm:overdamped-stationarity}).
(ii) For the critically damped reverse process we prove an averaged bound on the
full phase-space relative Fisher divergence
(Theorem~\ref{thm:underdamped-stationarity}), with the sampler compared to the
twin chain that pushes the \emph{exact} initialization through the \emph{same}
numerical kernels, so that the bound isolates the initialization error.

\paragraph{Organization.}

Section~\ref{sec:setting} fixes the setting; in Section~\ref{sec:decay}, we give results for continuous reverse flows. In Sections~\ref{sec:overdamped-stat} and \ref{sec:underdamped-stat}, we give results for discrete samplers of the overdamped and underdamped model separately. Section~\ref{sec:experiments} gives our experiment results. Appendices~\ref{app:setting-proofs}--\ref{app:ud-proofs} prove the results of Sections~\ref{sec:setting}--\ref{sec:underdamped-stat} in order, and Appendix~\ref{app:critical} collects critical-damping identities. Our work focuses on the process mechanism itself, so we analyze initialization error and discretization error. Throughout the paper we assume we have access to the true score functions.

\section{Setting: two reverse diffusions and a first-order criterion}
\label{sec:setting}

This section introduces the two forward models and their reverse dynamics.

\paragraph{Overdamped model.}
The forward overdamped Langevin diffusion on $\R^d$ is
\begin{equation}
    \dd X_s=-\nabla U(X_s)\dd s+\sqrt2\,\dd W_s,
    \qquad s\in[0,T],
    \label{eq:od-forward}
\end{equation}
with invariant density $p^*\propto e^{-U}$. Its Fokker--Planck equation can be rewritten as the continuity equation $\partial_s\rho_s=\nabla\cdot(\rho_s\nabla\log(\rho_s/p^*))$: the probability mass of $\rho_s$ is transported by the velocity field $-\nabla\log(\rho_s/p^*)$. Running time backward, the exact reverse marginals $q_t:=\rho_{T-t}$ therefore satisfy $\partial_tq_t=-\nabla\cdot(q_tv_t^q)$ with \emph{current velocity}
\begin{equation}
    v_t^q=\nabla\log\frac{q_t}{p^*},
    \label{eq:od-current}
\end{equation}
realized by the exact reverse SDE $\dd Y_t=[\nabla U(Y_t)+2\nabla\log q_t(Y_t)]\dd t+\sqrt2\,\dd\bar W_t$ \citep{anderson1982reverse,haussmann1986time}. If $p_t$ is the law of the \emph{same} reverse SDE initialized from $p_0\ne q_0$, then $\partial_tp_t=-\nabla\cdot(p_tv_t^p)$ with
\begin{equation}
    v_t^p=v_t^q-\nabla r_t,
    \qquad
    r_t=\log\frac{p_t}{q_t}.
    \label{eq:od-vq}
\end{equation}

\paragraph{Underdamped model.}
The forward kinetic (underdamped) Langevin diffusion lives on phase space $\R^d\times\R^d$, in variables $z=(x,v)$, with friction $\gamma>0$ and inverse mass $\xi>0$:
\begin{equation}
    \dd X_s=\xi V_s\dd s,
    \qquad
    \dd V_s=-\nabla U(X_s)\dd s-\gamma\xi V_s\dd s+\sqrt{2\gamma}\,\dd W_s,
    \label{eq:ud-forward}
\end{equation}
with invariant density $p^*\propto\exp(-U(x)-\tfrac\xi2\|v\|^2)$. Noise and friction act on the velocity coordinates only. Write
$J=\big(\begin{smallmatrix}0&\Id\\-\Id&0\end{smallmatrix}\big)$ and
$G_\gamma=\big(\begin{smallmatrix}0&0\\0&\gamma \Id\end{smallmatrix}\big)$,
so that the drift of \eqref{eq:ud-forward} is $b_{\mathrm{fwd}}(z)=-(J-G_\gamma)\nabla\log p^*=(\xi v,\,-\nabla U(x)-\gamma\xi v)$, with Jacobian
\begin{equation}
    M(x):=\nabla b_{\mathrm{fwd}}(z)
    =\begin{pmatrix}0&\xi\Id\\-\nabla^2U(x)&-\gamma\xi \Id\end{pmatrix}
    \label{eq:M-def}
\end{equation}
The Fokker--Planck equation of \eqref{eq:ud-forward} is again a continuity equation, with current velocity $w_s=-J\nabla\log p^*-G_\gamma\nabla\log(\rho_s/p^*)$ (Lemma~\ref{lem:fwd-current}); hence $q_t:=\rho_{T-t}$ obeys $\partial_tq_t=-\nabla\cdot(q_tv_t^q)$ with
\begin{equation}
    v_t^q=J\nabla\log p^*+G_\gamma\nabla\log\frac{q_t}{p^*},
    \label{eq:ud-current}
\end{equation}
realized by the exact reverse SDE
\begin{equation}
    \dd\bar X_t=-\xi\bar V_t\dd t,
    \qquad
    \dd\bar V_t=\big[\nabla U(\bar X_t)+\gamma\xi\bar V_t
    +2\gamma\nabla_v\log q_t(\bar X_t,\bar V_t)\big]\dd t
    +\sqrt{2\gamma}\,\dd\bar W_t,
    \label{eq:ud-reverse}
\end{equation}
in which only the \emph{velocity} score $\nabla_v\log q_t$ enters---one reason kinetic diffusion models are attractive in practice. For $p_t$ solving the same reverse SDE from $p_0\ne q_0$
\begin{equation}
    v_t^p=v_t^q-G_\gamma\nabla r_t,
    \qquad
    r_t=\log\frac{p_t}{q_t}.
    \label{eq:ud-vq}
\end{equation}
The same probability flow \eqref{eq:ud-current} is also realized by the deterministic reverse ODE
\begin{equation}
    \dd\bar X_t=-\xi\bar V_t\dd t,
    \qquad
    \dd\bar V_t=\big[\nabla U(\bar X_t)+\gamma\xi\bar V_t
    +\gamma\nabla_v\log q_t(\bar X_t,\bar V_t)\big]\dd t
    \label{eq:ud-ode}
\end{equation}

\subsection{Divergences and the exact dissipation identity}
This subsection defines the objects of the dictionary of Table~\ref{tab:dictionary}, and then proves the analog of the descent identities, which holds with no assumption on $U$ or on the data.
For smooth positive probability densities $p,q$ on $\R^n$ ($n=d$ or $2d$), write $r=\log(p/q)$, $g=\nabla r$, and define
\begin{equation}
    \KL(p\,\|\,q)=\int p\log\frac pq,
    \qquad
    \F(p\,\|\,q)=\int p\,\|g\|^2,
    \qquad
    \F_P(p\,\|\,q)=\int p\,\ip{g}{Pg},
    \label{eq:divergences}
\end{equation}
for a symmetric positive definite matrix $P$. In the underdamped model we also use the velocity part $\Fv(p\,\|\,q)=\int p\,\|\nabla_vr\|^2$.

\begin{lemma}[KL derivative for two continuity equations]
\label{lem:kl-identity}
Suppose $\partial_tq_t=-\nabla\cdot(q_tv_t^q)$ and $\partial_tp_t=-\nabla\cdot(p_tv_t^p)$, and let $r_t=\log(p_t/q_t)$. Then
\begin{equation}
    \frac{\dd}{\dd t}\KL(p_t\,\|\,q_t)=\int p_t\,\nabla r_t\cdot(v_t^p-v_t^q)
\end{equation}
\end{lemma}

Combining Lemma~\ref{lem:kl-identity} with the velocity differences in \eqref{eq:od-vq} and \eqref{eq:ud-vq} gives the reverse-diffusion descent identity.

\begin{corollary}[KL dissipation along the exact reverse flows]
\label{cor:kl-dissipation}
Along the common exact reverse dynamics,
\[
    \frac{\dd}{\dd t}\KL(p_t\|q_t)
    =-\F(p_t\|q_t)
    \;\;\text{(overdamped)},
    \quad
    \frac{\dd}{\dd t}\KL(p_t\|q_t)
    =-\gamma\,\Fv(p_t\|q_t)
    \;\;\text{(underdamped)}.
\]
In particular $\KL(p_t\,\|\,q_t)$ is always nonincreasing in both models.
\end{corollary}

The entire effect of the mismatched initialization is the single extra correction term $v_t^p-v_t^q$. The probability-flow ODE
realizes the same marginals $q_t$ from $q_0$, but transports \emph{every} initial law by the
same field $v_t^q$, so no correction term appears at all.

\begin{proposition}[No forgetting along the probability-flow ODE]
\label{prop:ode-transport}
Let $\Phi_t$ be the flow map of the probability-flow ODE of either model, and let
$p_t=(\Phi_t)_\#p_0$ for an arbitrary initial law $p_0$, and $q_t=(\Phi_t)_\#q_0$ are the
exact reverse marginals. Then $r_t=r_0\circ\Phi_t^{-1}$; hence for all $t\in[0,T]$
\[
\KL(p_t\,\|\,q_t)=\KL(p_0\,\|\,q_0),\qquad \mathrm{TV}(p_t,q_t)=\mathrm{TV}(p_0,q_0),
\]
and likewise for every $f$-divergence, while the relative score is merely carried along and
rescaled by the Jacobian of the flow,
\[
\nabla r_t\big(\Phi_t(x)\big)=\big(\nabla\Phi_t(x)\big)^{-\top}\nabla r_0(x),\qquad
\F(p_t\,\|\,q_t)=\int p_0\,\big\|(\nabla\Phi_t)^{-\top}\nabla r_0\big\|^2 .
\]
\end{proposition}
In the dictionary the ODE is a change of variables: the objective and every zeroth-order
distance are untouched, and the gradient is transported, not contracted.

\begin{remark}[Optimization counterpart, and a moving target]
\label{rem:optimization-counterpart}
Corollary~\ref{cor:kl-dissipation} has a familiar optimization counterpart.
For the gradient flow $\dot x_t=-\nabla f(x_t)$ and for the heavy-ball flow
$\dot x_t=v_t$, $\dot v_t=-\nabla f(x_t)-\gamma v_t$, one has
$
  \frac{d}{dt}\,f(x_t)=-\|\nabla f(x_t)\|^2
$ and
$
  \frac{d}{dt}\Bigl[f(x_t)+\tfrac12\|v_t\|^2\Bigr]=-\gamma\,\|v_t\|^2 .
$
In Wasserstein space the role
of $f$ is played by $\mathrm{KL}(\cdot\,\|\, q_t)$. 
In reverse diffusion the target itself moves: $p_t$ chases $q_t$,
which travels on a schedule fixed by the forward process; for the
exact flows this costs nothing, and Lemma~\ref{lem:kl-identity}
shows why. The same terms in $v^q_t$ and $v^p_t$ only make two distributions move in parallel; only relative velocity $v^p_t-v^q_t$ matters, and it is precisely what drives $p_t$ toward $q_t$. 
As in nonconvex optimization, a counterpart of the averaged statement
$\frac1T\int_0^T\|\nabla f(x_t)\|^2\,dt\le\bigl(f(x_0)-\inf f\bigr)/T$
can be obtained by integrating
Corollary~\ref{cor:kl-dissipation}. Denote the initialization
error as $H_0:=\KL(p_0\,\|\,q_0)$.
\end{remark}

\begin{proposition}[Averaged stationarity of the exact reverse flows]
\label{prop:integrated-dissipation}
For every potential $U$, every data distribution, and every pair of initializations $(p_0,q_0)$,
\begin{equation}
    \frac{1}{T}\int_0^T\F(p_t\,\|\,q_t)\dd t
    =\frac{H_0-\KL(p_T\,\|\,q_T)}{T}
    \le \frac{H_0}{T}
    \qquad\text{(overdamped)},
    \label{eq:integrated-dissipation}
\end{equation}
\begin{equation}
    \frac{\gamma}{T}\int_0^T\Fv(p_t\,\|\,q_t)\dd t= \frac{H_0-\KL(p_T\,\|\,q_T)}{T}\le \frac{H_0}{T}
    \qquad\text{(underdamped)},
    \label{eq:integrated-dissipation-ud}
\end{equation}
\end{proposition}

Assuming only $H_0<+\infty $, the time-averaged Fisher divergence of the overdamped flow converges to $0$ as $T\rightarrow +\infty$. However, for the underdamped model we only obtain a bound for the divergence in velocity dimensions $\F_v$. The theorem of Section~\ref{sec:overdamped-stat} is a discrete-time version of Proposition~\ref{prop:integrated-dissipation} for the overdamped model.

\section{Exponential Fisher contraction of the exact reverse flows}
\label{sec:decay}

In this section we directly study the evolution of $\F(p_t\,\|\,q_t)$, and show that it contracts exponentially fast under strong convexity of the forward potential $U$. Throughout this section $q_t$ is the exact reverse marginal, $p_t$ solves the \emph{same} exact reverse SDE from an arbitrary initial law $p_0$, and $r_t=\log(p_t/q_t)$, $g_t=\nabla r_t$, $S_t=\nabla^2r_t$.

\begin{assumption}[Strong convexity and smoothness]
\label{ass:convexity}
$U\in C^2(\R^d)$ and there are $0<m\le L<\infty$ with
$m\Id\preceq\nabla^2U(x)\preceq L\Id$ for all $x$.
\end{assumption}

Assumption~\ref{ass:convexity} constrains only the potential $U$ chosen by the practitioner (for the standard Gaussian noising process $m=L=1$).

\paragraph{Overdamped: the ordinary Fisher divergence contracts.}
The mechanism is a dissipation identity for $\F(t):=\F(p_t\,\|\,q_t)$, proved in Appendix~\ref{app:decay-proofs} (Proposition~\ref{prop:od-identity}): along the exact reverse flow,
\begin{equation}
    \frac{\dd}{\dd t}\F(t)
    =-2\int p_t\,\norm{S_t}_{\mathrm F}^2
    -2\int p_t\,g_t^\top\nabla^2U\,g_t.
    \label{eq:od-dissipation}
\end{equation}
The first term is a negative term and is always favorable; the second is where convexity of $U$ enters. Gronwall's inequality then gives:

\begin{theorem}[Overdamped Fisher contractivity]
\label{thm:overdamped-decay}
In the overdamped model, if $U$ is convex, then $\F(p_t\,\|\,q_t)$ is nonincreasing; if $U$ is $m$-strongly convex, then
$\F(p_t\,\|\,q_t)\le e^{-2mt}\,\F(p_0\,\|\,q_0)$ for $0\le t\le T$.
\end{theorem}

\begin{remark}[The reverse flow inherits forward contractivity]
\label{rem:od-inherit}
Perhaps surprisingly, identity \eqref{eq:od-dissipation} is exactly the same as for the \emph{forward} Langevin dynamics: Forward Fisher contractivity against the invariant law relies on convexity of $U$ alone, and the reverse flow inherits it \emph{without any functional inequality on $q_t$}. This is a special case, which we want to highlight here for its importance, of Fisher evolution along isotropic Fokker--Planck channels (\citealp[Lemma~2]{wibisono2025mixing}). For continuous reverse processes, our technical contribution is the theorem for the underdamped model.
\end{remark}

\paragraph{Underdamped: a modified Fisher divergence contracts.}
The underdamped flow dissipates only the velocity components of the relative score (Corollary~\ref{cor:kl-dissipation}), so position-direction error can be damped only \emph{indirectly}, by rotating into the velocity directions through the Hamiltonian coupling. Following the hypocoercivity paradigm \citep{villani2009hypocoercivity}, we therefore measure the error in a quadratic form with position--velocity \emph{cross terms}, chosen so that the coupling converts the conserved part of the error into the dissipated part. Concretely, for $\gamma^2\xi>L$ define
\begin{equation}
    P_\gamma=\begin{pmatrix}1&-\dfrac\gamma2\\[4pt]-\dfrac\gamma2&\dfrac{\gamma^2}2\end{pmatrix},
    \qquad
    \kappa_{\gamma,\xi}=\min\Big(\frac m\gamma,\;\gamma\xi-\frac L\gamma\Big)>0.
    \label{eq:explicit-P}
\end{equation}
With $P=P_\gamma\otimes\Id$, the associated \emph{modified Fisher divergence} is
\begin{equation}
    \F_P(p_t\,\|\,q_t)=\int p_t\,g_t^\top P\,g_t\dd z
    =\int p_t\Big(\|\nabla_xr_t\|^2-\gamma\ip{\nabla_xr_t}{\nabla_vr_t}+\frac{\gamma^2}2\|\nabla_vr_t\|^2\Big)\dd z .
    \label{eq:IP-def}
\end{equation}

\begin{theorem}[Underdamped modified Fisher contractivity]
\label{thm:underdamped-decay}
Assume the underdamped model \eqref{eq:ud-forward}, Assumption~\ref{ass:convexity}, and $\gamma^2\xi>L$. With $P=P_\gamma\otimes\Id$ as in \eqref{eq:explicit-P},
\[
    \F_P(p_t\,\|\,q_t)\le e^{-2t\kappa_{\gamma,\xi} }\,\F_P(p_0\,\|\,q_0)
\]
\end{theorem}

Since $P$ has eigenvalues bounded away from $0$ and $\infty$, contraction of $\F_P$ is equivalent, up to a fixed conditioning factor, to contraction of $\F$ itself. For the standard Gaussian potential $U(x)=\tfrac12\|x\|^2$ we have $m=L=1$, and we can select parameters $\gamma=2$, $\xi=1$ used by critically-damped diffusion models \citep{dockhorn2022score}.

\section{Averaged stationarity of the discretized overdamped reverse}
\label{sec:overdamped-stat}

We now account for the error incurred by time discretization. The practically dominant forward process is the standard Ornstein--Uhlenbeck (OU) diffusion,
\begin{equation}
    \dd Y_s=-Y_s\dd s+\sqrt2\,\dd B_s,
    \qquad Y_0\sim\rho_0,
    \label{eq:forward-ou}
\end{equation}
with marginals $\rho_s$: the overdamped model \eqref{eq:od-forward} with $U(x)=\tfrac12\|x\|^2$. The sampler we analyze is the \emph{frozen-score exponential integrator}, the idealized form of standard DDPM-type updates: on the uniform grid $t_k=kh$, $h=T/N$, freeze the score at the step start and solve the resulting linear SDE exactly. Its continuous-time interpolation is
\begin{equation}
    \dd X_t=\{X_t+2s_{t_k}(X_{t_k})\}\dd t+\sqrt2\,\dd B_t,
    \qquad t\in[t_k,t_{k+1}],
    \label{eq:numerical-sde}
\end{equation}
which integrates in closed form: for $\tau=t-t_k$,
\begin{equation}
    X_t=e^\tau X_{t_k}+2(e^\tau-1)s_{t_k}(X_{t_k})+\sqrt{e^{2\tau}-1}\,\zeta,
    \qquad \zeta\sim N(0,\Id).
    \label{eq:explicit-step}
\end{equation}
Let $\hp_t$ denote the law of $X_t$, the numerical process. The object of interest is the time-averaged Fisher divergence $\frac1T\int_0^T\F(\hp_t\,\|\,q_t)\dd t$ against the \emph{exact reverse marginal}: the sampling analog of the averaged gradient norm at a randomly stopped iterate \citep{ghadimi2013stochastic}, and the reverse-process counterpart of \citet{balasubramanian2022towards}.

\begin{assumption}[Score Lipschitzness]
\label{ass:lipschitz}
There is $L_s\ge0$ with $\|s_t(x)-s_t(y)\|\le L_s\|x-y\|$ for every $t\in[0,T]$ and $x,y\in\R^d$. 
\end{assumption}

\begin{assumption}[Finite initial divergence]
\label{ass:classical_od}
The initial laws have positive densities with
$H_0:=\KL(\hp_0\,\|\,q_0)<\infty$.
\end{assumption}

\begin{assumption}[Sub-Gaussian data]
\label{ass:expmoment}
For $Y_0\sim\rho_0$ there is a $\lambda>0$ with
$\E\exp\big(\tfrac{\lambda}{d}\|Y_0\|^2\big)=:M<\infty$.
\end{assumption}

Set $H_0:=\KL(\hp_0\,\|\,q_0)$, $\lambda_0:=\min\{\lambda,\tfrac d4\}$, and $\mathcal M:=\frac{d}{\lambda_0}(H_0+\log M+d)$.

\begin{remark}[Role of the Sub-Gaussian assumption]
\label{rem:expmoment-role}
Assumption~\ref{ass:expmoment} serves exactly one purpose in the proof: controlling the second moments of the \emph{numerical} marginals $\hp_t$. The assumption is equivalent to standard definition of sub-Gaussian distributions, normalized by the dimension $d$.
\end{remark}

\begin{theorem}[Averaged Fisher stationarity, overdamped]
\label{thm:overdamped-stationarity}
Suppose Assumptions~\ref{ass:lipschitz}--\ref{ass:expmoment} hold. There is a universal constant $c>0$ such that the following holds. If
\begin{equation}
    h\le\frac c{1+L_s}
    \qquad\text{and}\qquad
    T\big[L^2_sdh+L^2_s(1+L_s)^2\mathcal Mh^2\big]\le cd,
    \label{eq:bootstrap-condition}
\end{equation}
then the numerical law never drifts far from the exact one in relative entropy, $\sup_{0\le t\le T}\KL(\hp_t\,\|\,q_t)<H_0+d$, and
\begin{equation}
    \frac1T\int_0^T\F(\hp_t\,\|\,q_t)\dd t
    \;\lesssim\;
    \frac{H_0}T
    +L^2_sdh
    +L^2_s(1+L_s)^2\mathcal Mh^2.
    \label{eq:main-fisher}
\end{equation}
\end{theorem}

\begin{remark}[Shape of the bound]
\label{rem:od-shape}
Bound \eqref{eq:main-fisher} has exactly the shape of the nonconvex optimization guarantee in Table~\ref{tab:dictionary}: an amortized initialization term $H_0/T$ plus step-size terms. It refines the exact identity \eqref{eq:integrated-dissipation}, which it recovers as $h\to0$.
\end{remark}

\begin{remark}[Score error]
\label{rem:score-error-od}
Suppose the integrator is run with estimated scores $\hat s_{t_k}$ in place of $s_{t_k}$ in \eqref{eq:numerical-sde}, with
$\max_k\E_{\hp_{t_k}}\|\hat s_{t_k}-s_{t_k}\|^2\le\eps^2$. The only change in the proof is one triangle inequality inside the mismatch term of Proposition~\ref{prop:master},
$\E\|\hat s_{t_k}(X_{t_k})-s_t(X_t)\|^2\le2\eps^2+2\,\E\|s_{t_k}(X_{t_k})-s_t(X_t)\|^2$. Consequently, with slight adjustment of the proof, the conclusions of Theorem~\ref{thm:overdamped-stationarity} hold with the additional term $\eps^2$ on the right-hand side of \eqref{eq:main-fisher}. In practice, however, the
estimator is trained under the exact marginals $q_{t_k}$, so its error is
naturally controlled in $L^2(q_{t_k})$ rather than in $L^2(\hat p_{t_k})$.
Transferring the bound from one measure to the other requires either a uniform
bound or a change of measure, which is achievable under a Gaussian-tail assumption on score errors, but we do not pursue here. Appendix~\ref{app:experiments-ddpm} confirms empirically that the mechanism
analyzed here persists under a learned score.
\end{remark}

\section{Averaged stationarity of the discretized  critically damped reverse}
\label{sec:underdamped-stat}

We now prove the underdamped counterpart of Theorem~\ref{thm:overdamped-stationarity}, for the critically damped kinetic OU forward process: the underdamped model \eqref{eq:ud-forward} with $U(x)=\tfrac12\|x\|^2$ and the parameters $\gamma=2$, $\xi=1$, and for the frozen-score exponential integrator. One genuine obstacle separates this case from Section~\ref{sec:overdamped-stat}. Repeating the argument of Section~\ref{sec:overdamped-stat} along the frozen-score SDE interpolation is possible, but it could only provide a bound on $\Fv(\hp_t\,\|\,q_t)$, as in Proposition~\ref{prop:integrated-dissipation}. To upgrade $\Fv$ to the full $\F$ one wants a hypocoercive Lyapunov functional in the spirit of \citet{ma2021nesterov}; but differentiating $\F_P(\hp_t\,\|\,q_t)$ along an SDE interpolation like \eqref{eq:numerical-sde} produces derivatives of the score mismatch, objects for which no useful control is available.

Our resolution is decomposing the integrator. Since the frozen-score step solves a \emph{linear} SDE, it can be written as a deterministic \emph{score kick} applied at the beginning of the step, followed by the score-free reverse dynamics $\dd Y=CY\dd\tau+2E\dd W$ run for the full step (Lemma~\ref{lem:factorization}). Along this interpolation no derivative of the score error ever appears: the kick is a deterministic map, and the score-free leg is an exact linear diffusion on which entropy and twisted-Fisher identities hold in closed form. Therefore, hypocoercive dissipation and discretization decouple completely.

However, this decoupled composition also forces us to tell the story in an indirect way. Instead of bounding $\F(\hp_t\,\|\,q_t)$ directly, we will bound $\F(\hp_t\,\|\,\hq_t)$ for a reference process $\hq_t$, which isolates the initialization error. The discretization is isolated in the difference between $\hq_t$ and $q_t$, and it is controlled at the level of path measures by the standard Girsanov computation~\citep{chen2023sampling}. 

\subsection{The critically damped reverse process and the integrator}
\label{sec:ud-integrator2}

For $z=(x,v)\in\R^{2d}$ set
\begin{equation}
    C:=\begin{pmatrix}0&-\Id\\\Id&2\Id\end{pmatrix},
    \qquad
    E:=\begin{pmatrix}0\\\Id\end{pmatrix},
    \qquad
    \Pi_v:=EE^\top,
    \qquad
    D:=4\Pi_v .
    \label{eq:matrices}
\end{equation}
The critically damped kinetic noising process, in the forward clock $s$ and with marginal densities $\rho_s$, is
\begin{equation}
    \dd Z_s=-CZ_s\dd s+2E\,\dd B_s,
    \qquad\text{i.e.}\qquad
    \dd X_s=V_s\dd s,
    \qquad
    \dd V_s=-(X_s+2V_s)\dd s+2\,\dd B_s ;
    \label{eq:cdld-forward}
\end{equation}
its transition law over time $h$ is $N(e^{-hC}z,\,G_h)$ with $G_h:=\int_0^he^{-rC}De^{-rC^\top}\dd r$. In the \emph{increasing} reverse clock $t:=T-s$, with $q_t:=\rho_{T-t}$, the path-reversed process has dynamics
\begin{equation}
    \dd Y_t=\big(CY_t+4E\nabla_v\log q_t(Y_t)\big)\dd t+2E\,\dd\bar B_t,
    \qquad
    Y_0\sim\rho_T
    \label{eq:cdld-reverse}
\end{equation}
Fix an increasing reverse-time grid $0=t_0<t_1<\dots<t_N= T$ with steps $h_i:=t_{i+1}-t_i>0$, not necessarily uniform, and let $s_i:=\nabla_v\log q_{t_i}$ be the velocity score at reverse time $t_i$. The sampler is the frozen-score exponential integrator of \eqref{eq:cdld-reverse}, the kinetic analog of \eqref{eq:explicit-step}: freeze $s_i$ at the step start and solve the resulting linear SDE exactly,
\begin{equation}
    \hY_{i+1}=e^{h_iC}\hY_i
    +4\Big(\int_0^{h_i}e^{uC}\dd u\Big)Es_i(\hY_i)+\xi_i,
    \qquad
    \xi_i\sim\mathcal N(0,Q_{h_i}),
    \label{eq:update}
\end{equation}
where $Q_h:=\int_0^he^{rC}De^{rC^\top}\dd r=e^{hC}G_he^{hC^\top}$ is the covariance accumulated by the noise over a step of the reverse dynamics.
For $h>0$ define the kick vector and the kick map
\begin{equation}
    B_h:=4\int_0^{h}e^{-rC}E\dd r
    =\begin{pmatrix}4[1-(1+h)e^{-h}]\Id\\[2pt]4he^{-h}\Id\end{pmatrix},
    \qquad
    \Psi_{i,h}(z):=z+B_hs_i(z),
    \label{eq:B-Phi}
\end{equation}
and let $\mathcal R_\tau$ be the Markov semigroup of the \emph{score-free reverse dynamics}, the linear SDE
\begin{equation}
    \dd Y_\tau=CY_\tau\dd\tau+2E\,\dd W_\tau,
    \qquad
    \mathcal R_\tau:\ y\longmapsto e^{\tau C}y+\mathcal N(0,Q_\tau),
    \label{eq:score-free}
\end{equation}
obtained from \eqref{eq:cdld-reverse} by switching the score off. Therefore the transition kernel $\kernel_i$ of the update \eqref{eq:update} can be factorized as (Lemma~\ref{lem:factorization}):
\begin{equation}
    \mu\kernel_i=\mathcal R_{h_i}\big((\Psi_{i,h_i})_\#\mu\big)
    \qquad\text{for every input law }\mu :
    \label{eq:factorization-main}
\end{equation}
one step of the integrator is the score kick $\Psi_{i,h_i}$---the whole score impulse of the step, applied at its start---followed by the score-free reverse flow for the length of the step. All discretization error sits in the deterministic kick, and all stochastic evolution is an exact linear diffusion. We propagate the numerical initialization $\hp_0$ \emph{and} the exact initialization $\hq_0=\rho_T$ through the same kernels, interpolating within step $i$ by the clock $0\le\tau\le h_i$ of the score-free leg:
\begin{equation}
    \hp_{i+1}=\hp_i\kernel_i,
    \qquad
    \hq_{i+1}=\hq_i\kernel_i;
    \qquad
    \hp_{i,\tau}:=\mathcal R_\tau\big((\Psi_{i,h_i})_\#\hp_i\big),
    \qquad
    \hq_{i,\tau}:=\mathcal R_\tau\big((\Psi_{i,h_i})_\#\hq_i\big),
    \label{eq:interpolation}
\end{equation}
so that $\hp_{i,h_i}=\hp_{i+1}$ and $\hq_{i,h_i}=\hq_{i+1}$. At the grid points these are the laws of the sampler and of its twin; within a step they are the laws, at reverse time $t_i+\tau$, of the process that kicks at $t_i$ and then follows the score-free reverse dynamics.

\subsection{Assumptions and the stationarity theorem}
\label{sec:ud-assumptions}

\begin{assumption}[Score Lipschitzness]
\label{ass:score-map}
Every $s_i$ is $C^1$, and $\norm{\nabla s_i(z)}_\op\le L_{\mathrm{map}}$ for all $0\le i<N$ and $z\in\R^{2d}$, for some $L_{\mathrm{map}}\ge0$.
\end{assumption}

\begin{assumption}[Reference mixed Hessian]
\label{ass:reference}
$\norm{\nabla_z\nabla_v\log \hq_{i,\tau}(z)}_\op\le L_{\mathrm{ref}}$ for all $i$, $\tau\in(0,h_i]$, $z$, for some $L_{\mathrm{ref}}\ge0$.
\end{assumption}

\begin{assumption}[Finite initial divergence]
\label{ass:classical}
The initial laws have positive densities with
$H_0:=\KL(\hp_0\,\|\,q_0)<\infty$ and $F_0:=\F(\hp_0\,\|\,q_0)<\infty$.
\end{assumption}

\begin{remark}[On the status of Assumption~\ref{ass:reference}]
\label{rem:reference-status}
Assumption~\ref{ass:reference} is actually an implicit constraint on the data distribution, and does not involve the numerical initialization $\hp_0$. The twin $\hq_{i,\tau}$ is the \emph{exact} initialization $\rho_T$ pushed through the known, fixed kernels---kicks by the velocity scores and explicit linear diffusions---so Assumption~\ref{ass:reference} is an additional regularity property of the data propagated by known dynamics, and it can be explicitly verified for Gaussian $\rho_0$. 
\end{remark}

Set $\Lambda:=\max\{1,L_{\mathrm{map}},L_{\mathrm{ref}}\}$.

\begin{theorem}[Averaged relative-Fisher stationarity, underdamped]
\label{thm:underdamped-stationarity}
Suppose Assumptions~\ref{ass:score-map}--\ref{ass:classical} hold and
\begin{equation}
    \max_{0\le i<N}h_i\le\frac1{64\Lambda^2}.
    \label{eq:step-size}
\end{equation}
Then the full phase-space relative score satisfies
\begin{equation}
    \frac{1}{T}\sum_{i=0}^{N-1}\int_0^{h_i}\F(\hp_{i,\tau}\,\|\,\hq_{i,\tau})\dd\tau
    \;\lesssim\; \frac{\,\Lambda^2H_0+F_0}{T} .
    \label{eq:full-stationarity}
\end{equation}
\end{theorem}

\begin{remark}[Discretization error]
\label{rem:ud-discretization_error}
Notice that the right-hand side only contains initialization error. The discretization error is inside the difference between $\hq_{i,\tau}$ and $q_t$. Denote $Q_{0:T}$ the path measure of the exact reverse process and $\hQ_{0:T}$ that of the twin. The standard Girsanov analysis (Theorem 16 of \citet{chen2023sampling}) gives $\KL\left(Q_{0:T}\, \|\, \hQ_{0:T}  \right) \leq \tilde{O}(h)$ at grid times. 
\end{remark}

\section{Numerical experiments}
\label{sec:experiments}

\begin{figure}[t]
\centering
\includegraphics[width=\textwidth]{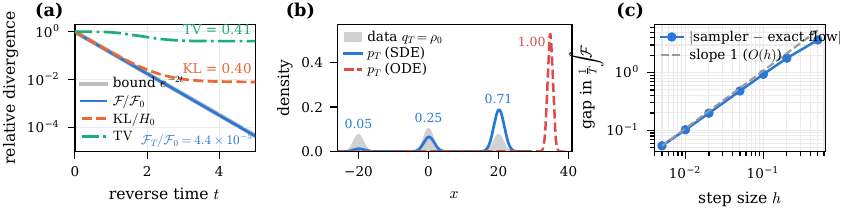}
\caption{One-dimensional Gaussian mixture with exact scores, wrong initialization $p_0=N(10,0.5^2)$. \textbf{(a)}~$\F/\F_0$ against the bound $e^{-2t}$ of Theorem~\ref{thm:overdamped-decay},
$\KL/H_0$, and $\mathrm{TV}$ along the exact reverse SDE. \textbf{(b)}~Terminal laws of the
reverse SDE and of the probability-flow ODE from the same initialization. \textbf{(c)}~Gap between the time-averaged $\F(\hp_t\,\|\,q_t)$ of the frozen-score integrator and its exact-flow counterpart $\F(p_t\,\|\,q_t)$ versus step size (log--log).}
\label{fig:gmm1d}
\end{figure}

Two experiments with \emph{exact} scores (Appendix~\ref{app:experiments}) show initialization error being forgotten exponentially fast in the first-order sense, but not the allocation of mass between well-separated modes.

\emph{First}, $\rho_0=0.3\,N(-20,1.5^2)+0.4\,N(0,1.5^2)+0.3\,N(20,1.5^2)$ is noised
by~\eqref{eq:forward-ou} up to $T=5$, and both exact reverse dynamics---the SDE and the
probability-flow ODE---are started from $p_0=N(10,0.5^2)$ instead of $q_0=\rho_T$
($H_0=49.8$, $\F_0=100$). Along the SDE (Figure~\ref{fig:gmm1d}(a)), $\F(p_t\,\|\,q_t)$
contracts at exactly the rate of Theorem~\ref{thm:overdamped-decay}, $\KL(p_t\,\|\,q_t)$
decreases monotonically but flattens once the trajectories have committed to a mode.
The terminal relative Fisher divergence $\F(p_t\,\|\, q_T) = 4.5*10^{-3}$ is small, but panel~(b) shows mode weights $(0.05, 0.25, 0.71)$ instead of $(0.3, 0.4, 0.3)$---locally indistinguishable from the
data, globally wrong. The ODE, with the same score, the same marginals and the same
initialization, forgets nothing (Proposition~\ref{prop:ode-transport}): the
terminal law is far outside the
data. Panel~(c): the time-averaged Fisher divergence of the integrator stays below $H_0/T$, with a gap to the exact flow decaying like $\tilde{O}(h)$ (Theorem~\ref{thm:overdamped-stationarity}).

\emph{Second}, with $\rho_0$ the empirical law of the $n=28000$ CelebA-HQ training images \citep{karras2018progressive} at $256\times256$, the OU marginals are kernel density estimates whose bandwidth is the noise level, so the sampler of Section~\ref{sec:overdamped-stat} runs with the \emph{exact} score ($T=5.06$, $150$ steps geometric in the noise-to-signal ratio, ending at $s=0$). Figure~\ref{fig:celeba} shows $64$ generated images for three initializations. Every image is close to a training image whatever the initialization---local consistency holds---but which images are produced depends on it: the nearly exact $N(0,\Id)$ returns a mixed draw of the training set, while the mean shifts $\mu=\pm4\,\mathbf 1/\sqrt d$ return almost exclusively bright, respectively dark, photographs. A pretrained DDPM, which includes the score estimation error and does not purely memorize, shows the same bias from the same initializations and horizon (Appendix~\ref{app:experiments-ddpm}).

\section{Conclusion}
\label{sec:discussion}

Reverse diffusions admit a first-order theory that parallels smooth optimization line by line. The relative score is the gradient; the SDE reverse flow is its descent dynamics, dissipating KL and contracting Fisher divergences exponentially under strong log-concavity of the forward stationary potential---a condition on the noising process one chooses, not on the data. Under mild assumptions, implementable integrators then achieve \emph{averaged} first-order stationarity at rate $H_0/T$ plus explicit step-size terms.

This also separates the two standard samplers. Only the SDE descends, and the ODE is a change of variables, so it forgets
nothing of its initialization.
Since the identities of Section~\ref{sec:setting} hold from any time on, the same holds for any
perturbation of the law introduced during sampling: the SDE forgets it in the first-order
sense, the ODE carries it to the end. This gives a precise form to the observation that
stochastic samplers correct earlier errors \citep{karras2022elucidating} and to the role of
Langevin corrector steps in ODE-based samplers \citep{song2021score,chen2023probability}.

However, the stationarity guarantee for the underdamped reverse process is indirect. Theorem~\ref{thm:underdamped-stationarity} compares the sampler $\hp$ with a twin chain $\hq$ that pushes the \emph{exact} initialization through the \emph{same} numerical kernels, so it isolates the initialization error. This is the price of decoupling the analysis of discretization and hypocoercivity, and Assumption~\ref{ass:reference} is also an artifact of this compromise. Even though we have $\KL\left(Q_{0:T}\, \|\, \hQ_{0:T}  \right) \leq \tilde{O}(h)$, it does not automatically transfer to a control on $\F(\hp_t\,\|\,q_t)$. Closing this gap---a first-order comparison of the kinetic sampler with the exact reverse process itself---is a natural direction for future work.

As the dictionary predicts, a first-order certificate is local: a law can have small relative Fisher divergence while misallocating mass between well-separated modes, so our guarantees certify local consistency with the target rather than recovery of mode weights---the price of convexity-free guarantees, in sampling as in nonconvex optimization. A log-Sobolev inequality for the data plays the role of the Polyak--{\L}ojasiewicz condition \citep{polyak1963gradient,karimi2016linear}: when it holds, the Fisher divergence bounds then convert into KL bounds against the reference laws: the first-order theory recovers zeroth-order convergence exactly when a PL-type inequality holds, as in optimization.

\bibliographystyle{plainnat}
\bibliography{references}

\appendix

\newcommand{\restatedhead}{}
\theoremstyle{plain}
\newtheorem*{restatedinner}{\restatedhead}
\newenvironment{restated}[3][]{\renewcommand{\restatedhead}{#2~#3}\begin{restatedinner}[#1]}{\end{restatedinner}}

\section*{Standing conventions for the appendices}

\emph{Regularity.} All probability densities appearing in the appendices are positive and smooth, with enough decay at infinity that the integrals we write are finite, that differentiation under the integral sign is justified, and that the integrations by parts we perform carry no boundary terms. This is the standard qualitative setting for Fisher-information computations \citep{vempala2019rapid,balasubramanian2022towards}.

\emph{Notation.} Universal constants $c,C>0$ are absolute and may change from line to line; $A\lesssim B$ means $A\le CB$ (in Appendix~\ref{app:ud-proofs}, $C$ is the matrix of \eqref{eq:matrices}). For a vector field $F$ we write $(\nabla F)_{ij}=\partial_{z_j}F_i$; for a scalar $\varphi$, $\nabla^2\varphi$ is the Hessian, and $\nabla_v$, $\Delta_v$ act on the velocity block. For the pair of laws under consideration we write $r=\log(\text{tracked law}/\text{reference law})$ and $g=\nabla r$, so that $\F(p\,\|\,q)=\int p\,\norm g^2$ and $\F_P(p\,\|\,q)=\int p\,\ip g{Pg}$ as in \eqref{eq:divergences}. 

\section{Proofs for Section~\texorpdfstring{\ref{sec:setting}}{2}}
\label{app:setting-proofs}

Both models of Section~\ref{sec:setting} are instances of a forward diffusion on $\R^n$ with constant diffusion matrix,
\begin{equation}
    \dd Z_s=b(Z_s)\dd s+\sqrt2\,G^{1/2}\dd W_s,
    \qquad s\in[0,T],
    \label{eq:generic-forward}
\end{equation}
where $b\in C^1(\R^n;\R^n)$ and $G$ is symmetric positive semidefinite: the overdamped model \eqref{eq:od-forward} has $n=d$, $b=-\nabla U$, $G=\Id$, and the underdamped model \eqref{eq:ud-forward} has $n=2d$, $b=b_{\mathrm{fwd}}$, $G=G_\gamma$. We write $\rho_s$ for the marginals of \eqref{eq:generic-forward} and treat both models at once.

\begin{lemma}[Current velocities of the forward and reverse flows]
\label{lem:generic-current}
The marginals of \eqref{eq:generic-forward} satisfy $\partial_s\rho_s=-\nabla\cdot(\rho_sw_s)$ with $w_s=b-G\nabla\log\rho_s$. Hence the reverse marginals $q_t:=\rho_{T-t}$ satisfy $\partial_tq_t=-\nabla\cdot(q_tv_t^q)$ with
\begin{equation}
    v_t^q=-b+G\nabla\log q_t,
    \label{eq:generic-current}
\end{equation}
and $q_t$ is the law of the reverse SDE
\begin{equation}
    \dd Y_t=\big[-b(Y_t)+2G\nabla\log q_t(Y_t)\big]\dd t+\sqrt2\,G^{1/2}\dd\bar W_t,
    \qquad Y_0\sim q_0 .
    \label{eq:generic-reverse}
\end{equation}
If $p_t$ is the law of \eqref{eq:generic-reverse} started from $p_0\ne q_0$ instead, then $\partial_tp_t=-\nabla\cdot(p_tv_t^p)$ with
\begin{equation}
    v_t^p=v_t^q-G\nabla r_t,
    \qquad r_t=\log\frac{p_t}{q_t}.
    \label{eq:generic-vp}
\end{equation}
\end{lemma}

\begin{proof}
The Fokker--Planck equation of \eqref{eq:generic-forward} is
\begin{align*}
    \partial_s\rho_s & = -\nabla\cdot(\rho_sb)+ \Delta(G\rho_s)\\
    &=- \nabla\cdot(\rho_sb)+\nabla\cdot(G\nabla\rho_s)\\
    &= - \nabla\cdot (\rho_s(b - G\nabla \log\rho_s))
\end{align*}

Substituting $s=T-t$ reverses the sign of the velocity, which gives \eqref{eq:generic-current}. The Fokker--Planck equation of \eqref{eq:generic-reverse}, for the law $\mu_t$ of a solution, is
\[
    \partial_t\mu_t
    =-\nabla\cdot\big(\mu_t[-b+2G\nabla\log q_t]\big)+\nabla\cdot(G\nabla\mu_t)
    =-\nabla\cdot\big(\mu_t[-b+2G\nabla\log q_t-G\nabla\log\mu_t]\big).
\]
For $\mu_t=q_t$ the velocity field in the bracket equals $v_t^q$, so $q_t$ solves this equation, and by uniqueness for the Fokker--Planck equation the SDE \eqref{eq:generic-reverse} started from $q_0$ has marginals $q_t$: this is the time-reversal theorem of \citet{anderson1982reverse,haussmann1986time}. For $\mu_t=p_t$ the bracket equals $v_t^q-G\nabla\log(p_t/q_t)$, which is \eqref{eq:generic-vp}.
\end{proof}

With $G=\Id$ and $b=-\nabla U$ one has $-b+G\nabla\log q_t=\nabla\log(q_t/p^*)$, and \eqref{eq:generic-current}--\eqref{eq:generic-vp} are \eqref{eq:od-current} and \eqref{eq:od-vq}. For the underdamped model, the next lemma records the form used in Section~\ref{sec:setting}.

\begin{lemma}[Current velocity, underdamped]
\label{lem:fwd-current}
Let $\rho_s$ be the density of \eqref{eq:ud-forward}. Then $\partial_s\rho_s=-\nabla\cdot(\rho_sw_s)$ with
\[
    w_s=-J\nabla\log p^*-G_\gamma\nabla\log\frac{\rho_s}{p^*} .
\]
Consequently $q_t=\rho_{T-t}$ has the current velocity \eqref{eq:ud-current}, and a mismatched initialization has the current velocity \eqref{eq:ud-vq}.
\end{lemma}

\begin{proof}
Since $\nabla\log p^*(z)=(-\nabla U(x),-\xi v)$, the drift of \eqref{eq:ud-forward} is $b_{\mathrm{fwd}}=(\xi v,-\nabla U(x)-\gamma\xi v)=-(J-G_\gamma)\nabla\log p^*$, and the noise $\sqrt{2\gamma}\,\dd W_s$ on the velocity block is $\sqrt2\,G_\gamma^{1/2}\dd W_s$. So \eqref{eq:ud-forward} is \eqref{eq:generic-forward} with $b=b_{\mathrm{fwd}}$ and $G=G_\gamma$, and by Lemma~\ref{lem:generic-current}, we have
\[
    w_s=b_{\mathrm{fwd}}-G_\gamma\nabla\log\rho_s
    =-J\nabla\log p^*+G_\gamma\nabla\log p^*-G_\gamma\nabla\log\rho_s
    =-J\nabla\log p^*-G_\gamma\nabla\log\frac{\rho_s}{p^*}.
\]
Then $v_t^q=-w_{T-t}$ is \eqref{eq:ud-current}, and \eqref{eq:generic-vp} with $G=G_\gamma$ is \eqref{eq:ud-vq}.
\end{proof}

\begin{restated}[KL derivative for two continuity equations]{Lemma}{\ref{lem:kl-identity}}
Suppose $\partial_tq_t=-\nabla\cdot(q_tv_t^q)$ and $\partial_tp_t=-\nabla\cdot(p_tv_t^p)$, and let $r_t=\log(p_t/q_t)$.
\begin{equation*}
    \frac{\dd}{\dd t}\KL(p_t\,\|\,q_t)=\int p_t\,\nabla r_t\cdot(v_t^p-v_t^q)
\end{equation*}
\end{restated}

\begin{proof}
By integration by parts and $\int\partial_tp_t=0$, 
\begin{align*}
    \frac{\dd}{\dd t}\int p_t\log\frac{p_t}{q_t}
    &=\int\partial_tp_t\,r_t-\int\frac{p_t}{q_t}\,\partial_tq_t \\
    &= -\int\nabla\cdot(p_tv_t^p)r_t + \int\nabla\cdot(q_tv_t^q) \frac{p_t}{q_t}\\
    & = \int p_t\,v_t^p\cdot\nabla r_t-\int q_tv_t^q\cdot\nabla(p_t/q_t)\\
    & = \int p_t\,v_t^p\cdot\nabla r_t-\int p_t\,v_t^q\cdot\nabla r_t 
\end{align*}
\end{proof}

\begin{restated}[KL dissipation along the exact reverse flows]{Corollary}{\ref{cor:kl-dissipation}}
Along the common exact reverse dynamics,
\[
    \frac{\dd}{\dd t}\KL(p_t\|q_t)=-\F(p_t\|q_t)\ \ \text{(overdamped)},
    \qquad
    \frac{\dd}{\dd t}\KL(p_t\|q_t)=-\gamma\,\Fv(p_t\|q_t)\ \ \text{(underdamped)}.
\]
In particular $\KL(p_t\,\|\,q_t)$ is always nonincreasing in both models.
\end{restated}

\begin{proof}
By \eqref{eq:generic-vp}, $v_t^p-v_t^q=-G\nabla r_t$, so Lemma~\ref{lem:kl-identity} gives $\frac{\dd}{\dd t}\KL(p_t\,\|\,q_t)=-\int p_t\,\nabla r_t^\top G\nabla r_t$. For $G=\Id$ this is $-\F(p_t\,\|\,q_t)$; for $G=G_\gamma$, which acts as $\gamma$ on the velocity block, it is $-\gamma\int p_t\|\nabla_vr_t\|^2=-\gamma\,\Fv(p_t\,\|\,q_t)$.
\end{proof}

\begin{restated}[No forgetting along the probability-flow ODE]{Proposition}{\ref{prop:ode-transport}}
Let $\Phi_t$ be the flow map of the probability-flow ODE of either model, and let $p_t=(\Phi_t)_\#p_0$ for an arbitrary initial law $p_0$, and $q_t=(\Phi_t)_\#q_0$ are the exact reverse marginals. Then $r_t=r_0\circ\Phi_t^{-1}$; hence for all $t\in[0,T]$
\[
    \KL(p_t\,\|\,q_t)=\KL(p_0\,\|\,q_0),\qquad \mathrm{TV}(p_t,q_t)=\mathrm{TV}(p_0,q_0),
\]
and likewise for every $f$-divergence, while the relative score is merely carried along and rescaled by the Jacobian of the flow,
\[
    \nabla r_t\big(\Phi_t(x)\big)=\big(\nabla\Phi_t(x)\big)^{-\top}\nabla r_0(x),
    \qquad
    \F(p_t\,\|\,q_t)=\int p_0\,\big\|(\nabla\Phi_t)^{-\top}\nabla r_0\big\|^2 .
\]
\end{restated}

\begin{proof}
\eqref{eq:twist-pullback} with $M=I$ gives the formula for $\F$. For a general $f$-divergence the same change of variables, $(p_t\circ\Phi_t)\,|\det\nabla\Phi_t|=p_0$ together with $r_t\circ\Phi_t=r_0$, yields
\[
    \int p_t\,f\Big(\frac{p_t}{q_t}\Big)
    =\int(p_t\circ\Phi_t)\,|\det\nabla\Phi_t|\,f\big(e^{r_0}\big)
    =\int p_0\,f\Big(\frac{p_0}{q_0}\Big),
\]
which for $f(u)=\log u$ is the KL statement, and for $f(u)=\frac12|\frac{1}{u}-1|$ is the TV statement. 
\end{proof}

\begin{restated}[Averaged stationarity of the exact reverse flows]{Proposition}{\ref{prop:integrated-dissipation}}
For every potential $U$, every data distribution, and every pair of initializations $(p_0,q_0)$, with $H_0=\KL(p_0\,\|\,q_0)$,
\begin{align*}
    \frac1T\int_0^T\F(p_t\,\|\,q_t)\dd t&=\frac{H_0-\KL(p_T\,\|\,q_T)}T\le\frac{H_0}T
    &&\text{(overdamped)},\\
    \frac\gamma T\int_0^T\Fv(p_t\,\|\,q_t)\dd t&=\frac{H_0-\KL(p_T\,\|\,q_T)}T\le\frac{H_0}T
    &&\text{(underdamped)}.
\end{align*}
\end{restated}

\begin{proof}
Integrate the identities of Corollary~\ref{cor:kl-dissipation} over $[0,T]$ (under the standing conventions $t\mapsto\KL(p_t\,\|\,q_t)$ is absolutely continuous) and use $\KL(p_T\,\|\,q_T)\ge0$.
\end{proof}

\section{Proofs for Section~\texorpdfstring{\ref{sec:decay}}{3}: exponential contraction}
\label{app:decay-proofs}

Throughout this appendix $q_t$ is the exact reverse marginal, $p_t$ solves the same exact reverse SDE from an arbitrary initial law $p_0$, and, as in Section~\ref{sec:decay}, $r_t=\log(p_t/q_t)$, $g_t=\nabla r_t$, $S_t=\nabla^2r_t$. Both contraction theorems come from one computation in the generic setting of Appendix~\ref{app:setting-proofs}: a transport equation for $g_t$ (Lemma~\ref{lem:master-transport}) and the dissipation identity it implies for quadratic forms of $g_t$ (Proposition~\ref{prop:master-dissipation}). 

\subsection{Transport of the relative score and the dissipation identity}
\label{app:master-identity}

In the setting of Lemma~\ref{lem:generic-current}, define the reference generator
\begin{equation}
    L_tf:=\frac1{q_t}\nabla\cdot(q_tG\nabla f)
    =G:\nabla^2f+a_t\cdot\nabla f,
    \qquad
    a_t:=G\nabla\log q_t ,
    \label{eq:generator-def}
\end{equation}
where $G:\nabla^2f=\operatorname{tr}(G\nabla^2f)$. It satisfies $\int q_t\,L_t\varphi=0$ by the divergence theorem and the regularity assumption; when it acts on a vector, it applies componentwise. It also follows the product rule
\begin{equation}
    L_t(\varphi\psi)=\varphi L_t\psi+\psi L_t\varphi+2\nabla\varphi^\top G\nabla\psi.
    \label{eq:product-rule}
\end{equation}
Indeed, since $\nabla(\varphi\psi)=\psi\nabla\varphi+\varphi\nabla\psi$ and
$\nabla\cdot(\psi V)=\psi\,\nabla\cdot V+\nabla\psi\cdot V$
for any vector field $V$,
\begin{align*}
    q_t\,L_t(\varphi\psi)&=\nabla\cdot\big(\psi\,q_tG\nabla\varphi\big)+\nabla\cdot\big(\varphi\,q_tG\nabla\psi\big)\\
    &=\psi\,\nabla\cdot(q_tG\nabla\varphi)+\varphi\,\nabla\cdot(q_tG\nabla\psi)
     +q_t\big(\nabla\psi^\top G\nabla\varphi+\nabla\varphi^\top G\nabla\psi\big),
\end{align*}

and the last bracket equals $2\nabla\varphi^\top G\nabla\psi$ because $G$ is symmetric;
dividing by $q_t$ yields \eqref{eq:product-rule}.

\begin{lemma}[Transport of the relative density and the relative score]
\label{lem:master-transport}
In the setting of Lemma~\ref{lem:generic-current},
\begin{align}
    (\partial_t+v_t^q\cdot\nabla)r_t&=L_tr_t+g_t^\top Gg_t,
    \label{eq:master-rt}\\
    (\partial_t+v_t^q\cdot\nabla)g_t
    &=L_tg_t+(\nabla b)^\top g_t+2S_tGg_t .
    \label{eq:master-gt}
\end{align}
\end{lemma}

\begin{proof}
Let $h_t=p_t/q_t=e^{r_t}$. Since $p_t\nabla r_t=\nabla p_t-p_t\nabla\log q_t=q_t\nabla h_t$, \eqref{eq:generic-vp} gives $p_tv_t^p=h_tq_tv_t^q-q_tG\nabla h_t$, so the continuity equation of $p_t$ can be written as
\[
    \partial_t(h_tq_t)=-\nabla\cdot(h_tq_tv_t^q)+\nabla\cdot(q_tG\nabla h_t).
\]
Expanding the left side and the first term on the right, and canceling $h_t\partial_tq_t=-h_t\nabla\cdot(q_tv_t^q)$, leaves $q_t\partial_th_t=-q_tv_t^q\cdot\nabla h_t+\nabla\cdot(q_tG\nabla h_t)$, that is,
\begin{equation}
    (\partial_t+v_t^q\cdot\nabla)h_t=L_th_t .
    \label{eq:h-transport}
\end{equation}

Now $h_t=e^{r_t}$. Since $\nabla e^{r}=e^{r}\nabla r$ and
$\nabla^2e^{r}=e^{r}\big(\nabla^2r+\nabla r\,\nabla r^\top\big)$, we have
\begin{align*}
    L_te^{r}
    &=G:\nabla^2e^{r}+a_t\cdot\nabla e^{r}\\
    &=e^{r}\big(G:\nabla^2r+\operatorname{tr}(G\,\nabla r\,\nabla r^\top)\big)+e^{r}\,a_t\cdot\nabla r\\
    &=e^{r}\big(L_tr+\nabla r^\top G\nabla r\big),
\end{align*}
where $\operatorname{tr}(G\,\nabla r\,\nabla r^\top)=\nabla r^\top G\nabla r$. On the left side of
\eqref{eq:h-transport}, $(\partial_t+v_t^q\cdot\nabla)e^{r_t}=e^{r_t}(\partial_t+v_t^q\cdot\nabla)r_t$
by the chain rule, so dividing \eqref{eq:h-transport} by $h_t=e^{r_t}$ yields \eqref{eq:master-rt}.

For \eqref{eq:master-gt} take the gradient of \eqref{eq:master-rt}. On the left of \eqref{eq:master-rt}, since
$\nabla$ and $\partial_t$ commute,
\begin{align*}
    \nabla\big[(\partial_t+v_t^q\cdot\nabla)r_t\big]
    &=\partial_t\nabla r_t+(\nabla v_t^q)^\top\nabla r_t+S_tv_t^q\\
    &=\partial_tg_t+v_t^q\cdot\nabla g_t+(\nabla v_t^q)^\top g_t\\
    &=(\partial_t+v_t^q\cdot\nabla)g_t+(\nabla v_t^q)^\top g_t .
\end{align*}
On the right of \eqref{eq:master-rt}, the second-order part $G:\nabla^2$ of $L_t$ has constant
coefficients, so it commutes with $\nabla$ and acts on $g_t$ componentwise; and $\nabla(g_t^\top Gg_t)=2S_tGg_t$, because
$\partial_j\sum_{i,k}(g_t)_iG_{ik}(g_t)_k=\sum_{i,k}\big[(S_t)_{ji}G_{ik}(g_t)_k+(g_t)_iG_{ik}(S_t)_{jk}\big]
=2(S_tGg_t)_j$ for symmetric $S_t$ and $G$. Hence
\begin{align*}
    \nabla\big[L_tr_t+g_t^\top Gg_t\big]
    &=\nabla\big(G:\nabla^2r_t\big)+\nabla(a_t\cdot\nabla r_t)+\nabla\big(g_t^\top Gg_t\big)\\
    &=G:\nabla^2g_t+\big[(\nabla a_t)^\top g_t+a_t\cdot\nabla g_t\big]+2S_tGg_t\\
    &=L_tg_t+(\nabla a_t)^\top g_t+2S_tGg_t .
\end{align*}
Equating the two expansions and moving $(\nabla v_t^q)^\top g_t$ to the right,
\[
    (\partial_t+v_t^q\cdot\nabla)g_t=L_tg_t+\big(\nabla(a_t-v_t^q)\big)^\top g_t+2S_tGg_t,
\]
and by \eqref{eq:generic-current}, $a_t-v_t^q=G\nabla\log q_t-(-b+G\nabla\log q_t)=b$, which is
\eqref{eq:master-gt}.
\end{proof}

The last line is the structural point of Section~\ref{sec:decay}: the coefficient $(\nabla b)^\top$ of the zeroth-order term in \eqref{eq:master-gt} does not depend on $q_t$, hence not on the data.

\begin{proposition}[Dissipation identity for quadratic forms of the relative score]
\label{prop:master-dissipation}
In the setting of Lemma~\ref{lem:generic-current}, for every constant symmetric matrix $P$,
\begin{equation}
    \frac{\dd}{\dd t}\int p_t\,g_t^\top Pg_t
    =\int p_t\,g_t^\top\big(\nabla b\,P+P\,(\nabla b)^\top\big)g_t
    -2\int p_t\operatorname{tr}\!\big(S_tGS_tP\big).
    \label{eq:master-identity}
\end{equation}
\end{proposition}

\begin{proof}
Write $F_t=g_t^\top Pg_t$ and $h_t=p_t/q_t$, so that $\int p_tF_t=\int q_t\,h_tF_t$. Since $q_t$ is transported by $v_t^q$, integrating its continuity equation by parts gives $\frac{\dd}{\dd t}\int q_t\varphi_t=\int q_t(\partial_t+v_t^q\cdot\nabla)\varphi_t$ for every time-dependent function $\varphi_t$. We apply this with $\varphi_t=h_tF_t$. Writing $D:=\partial_t+v_t^q\cdot\nabla$ (acting componentwise on
vector fields), the product rule, \eqref{eq:h-transport} and the chain rule for the quadratic form
$F_t=g_t^\top Pg_t$ give
\begin{align*}
    D(h_tF_t)
    &=F_t\,Dh_t+h_t\,DF_t\\
    &=F_t\,L_th_t+2h_t\,g_t^\top P\,Dg_t ,
\end{align*}
using that $P$ is symmetric. Multiplying by $q_t$, integrating, and using $q_th_t=p_t$ in the second term,

\begin{equation}
    \frac{\dd}{\dd t}\int p_tF_t
    =\int q_t\,(L_th_t)F_t
    +2\int p_t\,g_t^\top P\,(\partial_t+v_t^q\cdot\nabla)g_t .
    \label{eq:two-pieces}
\end{equation}
For the first term, the product rule \eqref{eq:product-rule} and $\int q_tL_t(h_tF_t)=0$ give
\[
    \int q_t\,(L_th_t)F_t
    =-\int q_t\,h_tL_tF_t-2\int q_t\,\nabla h_t^\top G\nabla F_t
    =-\int p_t\,L_tF_t-2\int p_t\,g_t^\top G\nabla F_t ,
\]
using $q_t\nabla h_t=p_tg_t$. Two derivatives of
$F_t=\sum_{i,k}(g_t)_iP_{ik}(g_t)_k$ are needed. First, $\nabla F_t=2S_tPg_t$.
Second, the product rule \eqref{eq:product-rule} applied to $(g_t)_i(g_t)_j$, with
$\nabla(g_t)_i=\nabla\partial_ir_t=S_te_i$ and hence $\nabla(g_t)_i^\top G\nabla(g_t)_j=(S_tGS_t)_{ij}$, yields
\begin{align*}
    L_tF_t
    &=\sum_{i,j}P_{ij}\Big[(g_t)_i\,L_t(g_t)_j+(g_t)_j\,L_t(g_t)_i+2\,(S_tGS_t)_{ij}\Big]\\
    &=2\,g_t^\top P\,L_tg_t+2\operatorname{tr}(S_tGS_tP),
\end{align*}
Inserting both into the previous display,
\[
    \int q_t\,(L_th_t)F_t
    =-2\int p_t\,g_t^\top PL_tg_t
    -2\int p_t\operatorname{tr}(S_tGS_tP)
    -4\int p_t\,g_t^\top GS_tPg_t .
\]
For the second term of \eqref{eq:two-pieces}, \eqref{eq:master-gt} shows
\[
    2\int p_t\,g_t^\top P(\partial_t+v_t^q\cdot\nabla)g_t
    =2\int p_t\,g_t^\top PL_tg_t
    +2\int p_t\,g_t^\top P(\nabla b)^\top g_t
    +4\int p_t\,g_t^\top PS_tGg_t .
\]
Adding the two, the terms in $L_tg_t$ cancel, and so do the two mixed terms, because $g^\top PSGg=g^\top GSPg$ for symmetric $P,S,G$. What remains is \eqref{eq:master-identity}, after writing $2g^\top P(\nabla b)^\top g=g^\top(\nabla b\,P+P(\nabla b)^\top)g$.
\end{proof}

\subsection{Overdamped}

\begin{proposition}[Dissipation identity for the ordinary Fisher divergence]
\label{prop:od-identity}
In the overdamped model, $\F(t):=\F(p_t\,\|\,q_t)$ satisfies \eqref{eq:od-dissipation}:
\[
    \F'(t)
    =-2\int p_t\,\norm{S_t}_{\mathrm F}^2
    -2\int p_t\,g_t^\top\nabla^2U\,g_t.
\]
\end{proposition}

\begin{proof}
Take $b=-\nabla U$, $G=\Id$ and $P=\Id$ in Proposition~\ref{prop:master-dissipation}: then $\nabla b\,P+P(\nabla b)^\top=-2\nabla^2U$ and $\operatorname{tr}(S_tS_t)=\norm{S_t}_{\mathrm F}^2$.
\end{proof}

\begin{restated}[Overdamped Fisher contractivity]{Theorem}{\ref{thm:overdamped-decay}}
In the overdamped model, if $U$ is convex, then $\F(p_t\,\|\,q_t)$ is nonincreasing; if $U$ is $m$-strongly convex, then
$\F(p_t\,\|\,q_t)\le e^{-2mt}\,\F(p_0\,\|\,q_0)$ for $0\le t\le T$.
\end{restated}

\begin{proof}
If $\nabla^2U\succeq0$, both terms in Proposition~\ref{prop:od-identity} are nonpositive. If $\nabla^2U\succeq m\Id$, the identity gives $\F'(t)\le-2m\F(t)$, and Gronwall's inequality concludes.
\end{proof}

\subsection{Underdamped}

In the underdamped model the forward drift is $b_{\mathrm{fwd}}$, whose Jacobian is $M(x)$ in \eqref{eq:M-def}, and $G=G_\gamma$. Proposition~\ref{prop:master-dissipation} therefore gives, for every constant symmetric $P$,
\begin{equation}
    \frac{\dd}{\dd t}\F_P(p_t\,\|\,q_t)
    =\int p_t\,g_t^\top\big(M(x)P+PM(x)^\top\big)g_t\dd z
    -2\int p_t\operatorname{tr}\!\big(S_tG_\gamma S_tP\big)\dd z.
    \label{eq:IP-derivative}
\end{equation}
The second term is nonpositive whenever $P\succ0$, since $\operatorname{tr}(S_tG_\gamma S_tP)=\|G_\gamma^{1/2}S_tP^{1/2}\|_{\mathrm F}^2$. The first term is where the Hamiltonian coupling acts: $M(x)^\top g_t$ rotates position-gradient error into the velocity directions, and everything comes down to finding a $P$ for which $M(x)P+PM(x)^\top$ is negative definite uniformly in $x$. The matrix $P_\gamma$ of \eqref{eq:explicit-P} does this.

\begin{lemma}[Lyapunov inequality for $P_\gamma$]
\label{lem:explicit-P}
Suppose $\gamma^2\xi>L$ and let $P=P_\gamma\otimes\Id$ with $P_\gamma$, $\kappa_{\gamma,\xi}$ as in \eqref{eq:explicit-P}. Under Assumption~\ref{ass:convexity},
\begin{equation}
    M(x)P+PM(x)^\top\preceq-2\kappa_{\gamma,\xi}\,P
    \qquad\text{for all }x\in\R^d .
    \label{eq:LMI}
\end{equation}
Moreover $\det P_\gamma=\gamma^2/4>0$, so $P_\gamma\succ0$, with eigenvalues
\begin{equation}
    \lambda_\pm(P_\gamma)
    =\frac12\Big(1+\frac{\gamma^2}2\Big)\pm\frac12\sqrt{1+\frac{\gamma^4}4}.
    \label{eq:P-eigs}
\end{equation}
\end{lemma}

\begin{proof}
Fix $x$ and diagonalize $\nabla^2U(x)=O^\top\operatorname{diag}(\eta_1,\dots,\eta_d)\,O$ with $O$ orthogonal and $\eta_i\in[m,L]$. Conjugation by $\operatorname{diag}(O,O)$ leaves $P=P_\gamma\otimes\Id$ unchanged, because each $d\times d$ block of $P$ is a multiple of $\Id$, and turns $M(x)$ into $\big(\begin{smallmatrix}0&\xi\Id\\-\operatorname{diag}(\eta_i)&-\gamma\xi\Id\end{smallmatrix}\big)$, which splits into $d$ two-dimensional systems with matrices
\[
    M(\eta):=\begin{pmatrix}0&\xi\\-\eta&-\gamma\xi\end{pmatrix},
    \qquad\eta\in[m,L].
\]
So it suffices to prove $M(\eta)P_\gamma+P_\gamma M(\eta)^\top\preceq-2\kappa_{\gamma,\xi}P_\gamma$ for every $\eta\in[m,L]$.

In the coordinates $(x,\,x+\tfrac2\gamma v)$, i.e.\ after conjugation by $T=\big(\begin{smallmatrix}1&0\\1&2/\gamma\end{smallmatrix}\big)$, the matrix $M(\eta)$ becomes
\[
    A(\eta):=TM(\eta)T^{-1}
    =\begin{pmatrix}-\dfrac{\gamma\xi}2&\dfrac{\gamma\xi}2\\[6pt]\dfrac{\gamma\xi}2-\dfrac{2\eta}\gamma&-\dfrac{\gamma\xi}2\end{pmatrix},
    \qquad
    \frac{A(\eta)+A(\eta)^\top}2
    =\begin{pmatrix}-\dfrac{\gamma\xi}2&s\\[6pt]s&-\dfrac{\gamma\xi}2\end{pmatrix},
\]
with $s=(\gamma^2\xi-2\eta)/2\gamma$. The symmetric part has eigenvalues $-\gamma\xi/2\pm|s|$, and according to the sign of $s$,
\[
    -\frac{\gamma\xi}2+|s|=-\min\Big(\frac\eta\gamma,\;\gamma\xi-\frac\eta\gamma\Big)
    \le-\min\Big(\frac m\gamma,\;\gamma\xi-\frac L\gamma\Big)=-\kappa_{\gamma,\xi}
    \qquad(\eta\in[m,L]),
\]
which is negative precisely when $m>0$ and $\gamma^2\xi>L$. Thus $A(\eta)+A(\eta)^\top\preceq-2\kappa_{\gamma,\xi}I$. Multiplying on the left by $T^\top$ and on the right by $T$, and writing $Q:=T^\top T$, this reads $M(\eta)^\top Q+QM(\eta)\preceq-2\kappa_{\gamma,\xi}Q$; multiplying on both sides by $Q^{-1}$ gives $M(\eta)Q^{-1}+Q^{-1}M(\eta)^\top\preceq-2\kappa_{\gamma,\xi}Q^{-1}$. Finally
\[
    Q=\begin{pmatrix}2&2/\gamma\\2/\gamma&4/\gamma^2\end{pmatrix},
    \qquad
    Q^{-1}=\begin{pmatrix}1&-\gamma/2\\-\gamma/2&\gamma^2/2\end{pmatrix}=P_\gamma ,
\]
whose trace and determinant are $1+\gamma^2/2$ and $\gamma^2/4$; this implies \eqref{eq:P-eigs}.
\end{proof}

\begin{restated}[Underdamped modified Fisher contractivity]{Theorem}{\ref{thm:underdamped-decay}}
Assume the underdamped model \eqref{eq:ud-forward}, Assumption~\ref{ass:convexity}, and $\gamma^2\xi>L$. With $P=P_\gamma\otimes\Id$ as in \eqref{eq:explicit-P},
$\F_P(p_t\,\|\,q_t)\le e^{-2t\kappa_{\gamma,\xi}}\,\F_P(p_0\,\|\,q_0)$.
\end{restated}

\begin{proof}
In \eqref{eq:IP-derivative} with $P=P_\gamma\otimes\Id\succ0$, the second term is nonpositive and, by Lemma~\ref{lem:explicit-P}, the first is at most $-2\kappa_{\gamma,\xi}\int p_t\,g_t^\top Pg_t=-2\kappa_{\gamma,\xi}\F_P(p_t\,\|\,q_t)$. Gronwall's inequality concludes. If $U$ is merely convex ($m=0$), then we have $\frac{\dd}{\dd t}\F_P\le0$.
\end{proof}

\section{Proofs for Section~\texorpdfstring{\ref{sec:overdamped-stat}}{4}: the overdamped sampler}
\label{app:od-proofs}

Throughout this appendix, $q_t=\rho_{T-t}$ and $s_t=\nabla\log q_t$ are the exact reverse marginals and scores of the forward OU process \eqref{eq:forward-ou}, $X_t$ is the interpolated numerical process \eqref{eq:numerical-sde} on the uniform grid $t_k=kh$, $h=T/N$, and $\hp_t$ is its law. Assumption~\ref{ass:classical_od} supplies the positive densities and the finite initial entropy $H_0$ required by the standing conventions, and $\lambda_0=\min\{\lambda,\frac d4\}$, $\mathcal M=\frac d{\lambda_0}(H_0+\log M+d)$ are as in Section~\ref{sec:overdamped-stat}.

The proof of Theorem~\ref{thm:overdamped-stationarity} runs as follows. Freezing the score costs a source term in the entropy dissipation: for $t\in[t_k,t_{k+1}]$,
\begin{equation}
    \frac{\dd}{\dd t}\KL(\hp_t\,\|\,q_t)
    \le-\frac12\F(\hp_t\,\|\,q_t)
    +2\,\E\|s_{t_k}(X_{t_k})-s_t(X_t)\|^2
    \label{eq:master}
\end{equation}
(Proposition~\ref{prop:master}). The source term compares the frozen score with the current one. It has a spatial part, from the increment $X_t-X_{t_k}$, which Lipschitzness and the explicit step \eqref{eq:explicit-step} control, and a temporal part, from $s_{t_k}-s_t$ at a fixed point, which we control through the OU semigroup (Lemma~\ref{lem:temporal-score}; compare \citealp[Lemma~17]{chen2023sampling}). Both can be controlled by second moments under the numerical law $\hp_t$. Assumption~\ref{ass:expmoment} bounds exponential moments of the exact marginals $q_t$, and the Gibbs variational inequality converts them into second moments of $\hp_t$ as long as $\KL(\hp_t\,\|\,q_t)$ stays bounded (Lemma~\ref{lem:bootstrap-moments}). Integrating \eqref{eq:master} with this bound closes the entropy bootstrap and results in our main theorem.

\subsection{KL dissipation for the numerical marginal}
\label{app:od-master}

\begin{lemma}[Marginal Markov projection]
\label{lem:markov-projection}
For $t\in[t_k,t_{k+1}]$ let $\bar s_t(x):=\E[s_{t_k}(X_{t_k})\mid X_t=x]$. Then
\begin{equation}
    \partial_t\hp_t=-\nabla\cdot\big(\hp_t[x+2\bar s_t(x)]\big)+\Delta\hp_t.
    \label{eq:q-fp}
\end{equation}
\end{lemma}

\begin{proof}
For a smooth compactly supported $\varphi$, It\^o's formula applied to \eqref{eq:numerical-sde} gives
\begin{align*}
    \frac{\dd}{\dd t}\E\,\varphi(X_t)
    &=\E\big[\ip{\nabla\varphi(X_t)}{X_t+2s_{t_k}(X_{t_k})}+\Delta\varphi(X_t)\big]\\
    &=\E\big[\ip{\nabla\varphi(X_t)}{X_t+2\bar s_t(X_t)}+\Delta\varphi(X_t)\big],
\end{align*}
by conditioning on $X_t$ in the frozen-score term. This is the weak form of \eqref{eq:q-fp}.
\end{proof}

\begin{lemma}[Relative entropy derivative for two Fokker--Planck equations]
\label{lem:kl-derivative}
If $\partial_t\mu_t=-\nabla\cdot(\mu_tb_t^\mu)+\Delta\mu_t$ and $\partial_t\nu_t=-\nabla\cdot(\nu_tb_t^\nu)+\Delta\nu_t$, then
\begin{equation}
    \frac{\dd}{\dd t}\KL(\mu_t\,\|\,\nu_t)
    =-\F(\mu_t\,\|\,\nu_t)
    +\E_{\mu_t}\Big\langle\nabla\log\frac{\mu_t}{\nu_t},\,b_t^\mu-b_t^\nu\Big\rangle.
    \label{eq:kl-derivative}
\end{equation}
\end{lemma}

\begin{proof}
Both equations are continuity equations, with velocities $v_t^\mu=b_t^\mu-\nabla\log\mu_t$ and $v_t^\nu=b_t^\nu-\nabla\log\nu_t$. Thus $v_t^\mu-v_t^\nu=b_t^\mu-b_t^\nu-\nabla\log(\mu_t/\nu_t)$, and Lemma~\ref{lem:kl-identity} completes the proof.
\end{proof}

\begin{proposition}[Master inequality]
\label{prop:master}
For $t\in[t_k,t_{k+1}]$,
\[
    \frac{\dd}{\dd t}\KL(\hp_t\,\|\,q_t)
    \le-\frac12\F(\hp_t\,\|\,q_t)
    +2\,\E\|s_{t_k}(X_{t_k})-s_t(X_t)\|^2 .
\]
\end{proposition}

\begin{proof}
The exact marginal satisfies $\partial_tq_t=-\nabla\cdot(q_t[x+2s_t])+\Delta q_t$, the Fokker--Planck equation of the exact reverse SDE (Lemma~\ref{lem:generic-current} with $b=-x$, $G=\Id$). Lemma~\ref{lem:kl-derivative} with $\mu_t=\hp_t$, $\nu_t=q_t$ and Lemma~\ref{lem:markov-projection} give
\[
    \frac{\dd}{\dd t}\KL(\hp_t\,\|\,q_t)
    =-\F(\hp_t\,\|\,q_t)+2\,\E_{\hp_t}\Big\langle\nabla\log\frac{\hp_t}{q_t},\,\bar s_t-s_t\Big\rangle
    \le-\frac12\F(\hp_t\,\|\,q_t)+2\,\E_{\hp_t}\|\bar s_t-s_t\|^2 ,
\]
by $2\ip ge\le\frac12\|g\|^2+2\|e\|^2$. Finally $\bar s_t(X_t)-s_t(X_t)=\E[s_{t_k}(X_{t_k})-s_t(X_t)\mid X_t]$, and Jensen's inequality for the conditional expectation bounds $\E\|\bar s_t(X_t)-s_t(X_t)\|^2$ by $\E\|s_{t_k}(X_{t_k})-s_t(X_t)\|^2$.
\end{proof}

\subsection{Temporal score variation from the OU semigroup}
\label{app:od-temporal}

Along the forward flow \eqref{eq:forward-ou}, $Y_{s+\tau}\overset{d}=e^{-\tau}Y_s+\sqrt{1-e^{-2\tau}}\,G$ with $G\sim N(0,\Id)$ independent of $Y_s$. Read at $s=T-t$ and $s+\tau=T-t_k$, i.e.\ at $t=t_k+\tau$, this expresses the earlier reverse-time marginal as a scaled and then smoothed version of the later one:
\begin{equation}
    q_{t_k}=\big((a\Id)_\#q_t\big)*N(0,\sigma^2\Id),
    \qquad
    a=e^{-\tau},
    \quad
    \sigma^2=1-e^{-2\tau}.
    \label{eq:ou-local-relation}
\end{equation}
We first quantify what the smoothing does to a score.

\begin{lemma}[Score perturbation by Gaussian smoothing]
\label{lem:gaussian-smoothing}
Let $\pi$ be a smooth positive density on $\R^d$ whose score $u=\nabla\log\pi$ is $L_\pi$-Lipschitz, and let $\pi_\sigma=\pi*N(0,\sigma^2\Id)$. If $L_\pi\sigma^2\le\frac12$, then for every $x$,
\begin{equation}
    \|\nabla\log\pi_\sigma(x)-u(x)\|
    \le L_\pi\sqrt{2d\sigma^2}+2L_\pi\sigma^2\,\|u(x)\|.
    \label{eq:gaussian-smoothing}
\end{equation}
\end{lemma}

\begin{proof}
Fix $x$ and let $\mu_x$ be the probability density on $\R^d$ proportional to $\pi(y)\exp(-\|x-y\|^2/2\sigma^2)$, the conditional law of $Y\sim\pi$ given $Y+\sigma G=x$. Since $\nabla_xe^{-\|x-y\|^2/2\sigma^2}=-\nabla_ye^{-\|x-y\|^2/2\sigma^2}$, an integration by parts in $y$ gives
\[
    \nabla\pi_\sigma(x)\propto\int\pi(y)\,\nabla_xe^{-\|x-y\|^2/2\sigma^2}\dd y
    =\int\nabla\pi(y)\,e^{-\|x-y\|^2/2\sigma^2}\dd y ,
\]
and dividing by $\pi_\sigma(x)$, with $\nabla\pi=\pi u$, gives $\nabla\log\pi_\sigma(x)=\E_{\mu_x}u(Y)$. Hence $\|\nabla\log\pi_\sigma(x)-u(x)\|\le L_\pi\,\E_{\mu_x}\|Y-x\|$, and it remains to show that $\mu_x$ concentrates near $x$.

Write $\mu_x\propto e^{-V_x}$ with $V_x(y)=-\log\pi(y)+\|x-y\|^2/2\sigma^2$. As $\nabla^2(-\log\pi)=-\nabla u\succeq-L_\pi\Id$, we have $\nabla^2V_x\succeq(\sigma^{-2}-L_\pi)\Id\succeq(2\sigma^2)^{-1}\Id$, so $V_x$ has a unique minimizer $y_\star$. At $y_\star$, we know
\[
    -u(y_\star)+(y_\star-x)/\sigma^2=0, \quad\text{i.e.} \quad\ y_\star-x=\sigma^2u(y_\star),
\]
and therefore
\[
    \|y_\star-x\|\le\sigma^2\|u(x)\|+L_\pi\sigma^2\|y_\star-x\|,
    \qquad\text{hence}\qquad
    \|y_\star-x\|\le2\sigma^2\|u(x)\| .
\]
Next we show that $\mu_x$ concentrates around $y_\star$. Since $\nabla\mu_x=-\mu_x\nabla V_x$, an
integration by parts shows,
\[
    \E_{\mu_x}\ip{Y-y_\star}{\nabla V_x(Y)}
    =-\int\ip{y-y_\star}{\nabla\mu_x(y)}\dd y
    =\int\mu_x(y)\,\nabla\cdot(y-y_\star)\dd y=d .
\]
On the other hand, we have $\nabla^2V_x\succeq(2\sigma^2)^{-1}\Id$ and $\nabla V_x(y_\star)=0$, so
\[
    \ip{y-y_\star}{\nabla V_x(y)}
    =\Big\langle y-y_\star,\int_0^1\nabla^2V_x\big(y_\star+s(y-y_\star)\big)(y-y_\star)\dd s\Big\rangle
    \ge\frac{\|y-y_\star\|^2}{2\sigma^2}.
\]
Therefore $\E_{\mu_x}\|Y-y_\star\|^2\le2d\sigma^2$.
Altogether, by the triangle inequality and Jensen's inequality,
\[
    \E_{\mu_x}\|Y-x\|
    \le\big(\E_{\mu_x}\|Y-y_\star\|^2\big)^{1/2}+\|y_\star-x\|
    \le\sqrt{2d\sigma^2}+2\sigma^2\|u(x)\|.
\]
Substituting this bound into $\|\nabla\log\pi_\sigma(x)-u(x)\|\le L_\pi\,\E_{\mu_x}\|Y-x\|$ completes the proof.
\end{proof}

\begin{lemma}[Temporal OU score perturbation]
\label{lem:temporal-score}
Under Assumption~\ref{ass:lipschitz}, let $t=t_k+\tau$ with $0\le\tau\le h$. If $h\le c/(1+L_s)$ with $c\le\frac18$, then for every $x\in\R^d$,
\begin{equation}
    \|s_{t_k}(x)-s_t(x)\|^2
    \lesssim
    L^2_sd\tau+L^2_s\tau^2\|x\|^2+(1+L_s)^2\tau^2\|s_t(x)\|^2.
    \label{eq:temporal-pointwise}
\end{equation}
\end{lemma}

\begin{proof}
Let $\pi=(a\Id)_\#q_t$ be the scaled density in \eqref{eq:ou-local-relation} and $u=\nabla\log\pi$. Since $\pi(x)=a^{-d}q_t(x/a)$,
\begin{equation}
    u(x)=a^{-1}s_t(x/a),
    \label{eq:scaled-score}
\end{equation}
so $u$ is $L_\pi$-Lipschitz with $L_\pi=L_sa^{-2}$. For $\tau\le h\le c/(1+L_s)$ we have $a^{-1}=e^\tau\le2$, $e^\tau-1\le2\tau$, $\sigma^2=1-e^{-2\tau}\le2\tau$, and $L_\pi\sigma^2=L_s(e^{2\tau}-1)\le4L_s\tau\le\frac12$. By  Lemma~\ref{lem:gaussian-smoothing}, we get
\begin{equation}
    \|s_{t_k}(x)-u(x)\|\lesssim L_s\sqrt{d\tau}+L_s\tau\|u(x)\| .
    \label{eq:smoothing-part}
\end{equation}
It remains to compare $u$ with $s_t$. By \eqref{eq:scaled-score}, $u(x)-s_t(x)=(a^{-1}-1)s_t(x/a)+\big(s_t(x/a)-s_t(x)\big)$, and $\|x/a-x\|=(e^\tau-1)\|x\|\le2\tau\|x\|$, so by Lipschitzness
\begin{equation}
    \|u(x)-s_t(x)\|
    \le2\tau\big(\|s_t(x)\|+2L_s\tau\|x\|\big)+2L_s\tau\|x\|
    \lesssim\tau\|s_t(x)\|+L_s\tau\|x\| .
    \label{eq:scaling-part}
\end{equation}
In particular $\|u(x)\|\lesssim\|s_t(x)\|+L_s\tau\|x\|$. Substituting this into \eqref{eq:smoothing-part}, adding \eqref{eq:scaling-part}, and using $L_s\tau\le1$,
\[
    \|s_{t_k}(x)-s_t(x)\|
    \lesssim L_s\sqrt{d\tau}+L_s\tau\|x\|+(1+L_s)\tau\|s_t(x)\| ,
\]
and squaring both sides completes the proof.
\end{proof}

\subsection{Second moments of the numerical law}
\label{app:od-moments}

\begin{lemma}[Exponential moments along the OU flow]
\label{lem:ou-expmoment}
Under Assumption~\ref{ass:expmoment}, with $\lambda_0=\min\{\lambda,\frac d4\}$,
\begin{equation}
    \E_{q_t}\exp\Big(\frac{\lambda_0}d\|X\|^2\Big)\le e^{2\lambda_0}M
    \qquad\text{for every }t\in[0,T].
    \label{eq:ou-expmoment}
\end{equation}
\end{lemma}

\begin{proof}
Every $q_t$ is a forward marginal $\rho_s$, and $Y_s\overset d=aY_0+\sigma G$ with $a=e^{-s}$, $\sigma^2=1-e^{-2s}$, and $G\sim N(0,\Id)$ independent of $Y_0\sim\rho_0$. Conditionally on $Y_0$, $Y_s$ is Gaussian with mean $aY_0$ and covariance $\sigma^2\Id$, and for $Z\sim N(m,\sigma^2\Id)$ and $2\theta\sigma^2<1$,
\[
    \E e^{\theta\|Z\|^2}=(1-2\theta\sigma^2)^{-d/2}\exp\Big(\frac{\theta\|m\|^2}{1-2\theta\sigma^2}\Big).
\]
Take $\theta=\lambda_0/d\le\frac14$ and put $x:=2\theta\sigma^2\le2\lambda_0/d\le\frac12$. Since $-\log(1-x)\le2x$ on $[0,\frac12]$, the prefactor is $(1-x)^{-d/2}\le e^{dx}\le e^{2\lambda_0}$, and $ a^2\le1-2\theta\sigma^2$ because $a^2=1-\sigma^2$ and $2\theta\le1$. Hence
\[
    \E e^{(\lambda_0/d)\|Y_s\|^2}
    \le e^{2\lambda_0}\,\E\exp\Big(\frac{\lambda_0a^2}{d(1-2\theta\sigma^2)}\|Y_0\|^2\Big)
    \le e^{2\lambda_0}\,\E e^{(\lambda_0/d)\|Y_0\|^2}\le e^{2\lambda_0}M,
\]
the last step because $\lambda_0\le\lambda$.
\end{proof}

\begin{lemma}[Gibbs variational inequality]
\label{lem:entropy-variational}
For probability measures $Q\ll P$ and every measurable $f$ with $\E_Pe^f<\infty$,
$\E_Qf\le\KL(Q\,\|\,P)+\log\E_Pe^f$.
\end{lemma}

\begin{proof}
With $\varrho=\dd Q/\dd P$, apply Jensen's inequality
\[\E_Qf-\KL(Q\,\|\,P)=\E_Q\log(e^f/\varrho)\le\log\E_Q(e^f/\varrho)\le\log\E_Pe^f.
\]
\end{proof}

\begin{lemma}[Moment bounds under the KL bootstrap]
\label{lem:bootstrap-moments}
Assume Assumptions~\ref{ass:lipschitz} and~\ref{ass:expmoment}. If, at some time $t$,
\begin{equation}
    \KL(\hp_t\,\|\,q_t)\le H_0+d,
    \label{eq:bootstrap-assumption}
\end{equation}
then
\begin{equation}
    \E_{\hp_t}\|X\|^2\lesssim\mathcal M,
    \qquad
    \E_{\hp_t}\|s_t(X)\|^2\lesssim L_s^2\mathcal M,
    \qquad
    \E_{\hp_t}\|X+2s_t(X)\|^2\lesssim(1+L_s)^2\mathcal M .
    \label{eq:q-moments}
\end{equation}
\end{lemma}

\begin{proof}
Lemma~\ref{lem:entropy-variational} with $f(x)=\frac{\lambda_0}d\|x\|^2$, Lemma~\ref{lem:ou-expmoment} and \eqref{eq:bootstrap-assumption} give
\[
    \frac{\lambda_0}d\,\E_{\hp_t}\|X\|^2
    \le\KL(\hp_t\,\|\,q_t)+\log\E_{q_t}e^{(\lambda_0/d)\|X\|^2}
    \le H_0+d+\log M+2\lambda_0
    \le\tfrac32\,\frac{\lambda_0}d\,\mathcal M
\]
which is the first bound; likewise we have
\[
\frac{\lambda_0}d\E_{q_t}\|Y\|^2\le\log\E_{q_t}e^{(\lambda_0/d)\|Y\|^2}\le\log M+2\lambda_0, 
\]
so $\E_{q_t}\|Y\|^2\le\frac d{\lambda_0}(\log M+\frac d2)\le\mathcal M$. For the score, $\E_{q_t}s_t(Y)=\int\nabla q_t=0$, so for every $x$,
\[
    \|s_t(x)\|=\big\|\E_{Y\sim q_t}[s_t(x)-s_t(Y)]\big\|\le L_s\,\E_{q_t}\|x-Y\|
    \le L_s\big(\|x\|+\E_{q_t}\|Y\|\big),
\]
hence $\|s_t(x)\|^2\le2L_s^2(\|x\|^2+\mathcal M)$. Averaging over $x\sim\hp_t$ yields the second bound, and $\|x+2s_t(x)\|^2\le2\|x\|^2+8\|s_t(x)\|^2$ the third.
\end{proof}

\subsection{One-step score error}
\label{app:od-onestep}

\begin{lemma}[Local score mismatch]
\label{lem:local-mismatch}
Assume Assumptions~\ref{ass:lipschitz} and~\ref{ass:expmoment}, and $h\le c/(1+L_s)$ with $c$ as in Lemma~\ref{lem:temporal-score}. Let $t=t_k+\tau$ with $0\le\tau\le h$, and suppose $\KL(\hp_s\,\|\,q_s)\le H_0+d$ for all $s\le t$. Then
\begin{equation}
    \E\|s_{t_k}(X_{t_k})-s_t(X_t)\|^2
    \lesssim
    L^2_sd\tau+L^2_s(1+L_s)^2\mathcal M\tau^2 .
    \label{eq:local-mismatch}
\end{equation}
\end{lemma}

\begin{proof}
We split
\begin{equation}
    \E\|s_{t_k}(X_{t_k})-s_t(X_t)\|^2
    \le2\,\E\|s_{t_k}(X_{t_k})-s_{t_k}(X_t)\|^2
    +2\,\E\|s_{t_k}(X_t)-s_t(X_t)\|^2 .
    \label{eq:mismatch-split}
\end{equation}
For the first term, by Lipschitzness we have
\[
\E\|s_{t_k}(X_{t_k})-s_{t_k}(X_t)\|^2\le L_s^2\,\E\|X_t-X_{t_k}\|^2, 
\]
and by \eqref{eq:explicit-step}
\[
    X_t-X_{t_k}=(e^\tau-1)\{X_{t_k}+2s_{t_k}(X_{t_k})\}+\sqrt{e^{2\tau}-1}\,\zeta
\]
with $\zeta\sim N(0,\Id)$ independent of $X_{t_k}$, so the cross term has zero mean. Since $e^\tau-1\le2\tau$ and $e^{2\tau}-1\le4\tau$ for $\tau\le h\le\frac18$, Lemma~\ref{lem:bootstrap-moments} at time $t_k$ implies
\begin{equation}
    \E\|X_t-X_{t_k}\|^2\lesssim(1+L_s)^2\mathcal M\tau^2+d\tau .
    \label{eq:increment}
\end{equation}
For the second term, take $x=X_t$ in Lemma~\ref{lem:temporal-score}, average over $X_t\sim\hp_t$, and use Lemma~\ref{lem:bootstrap-moments} at time $t$:
\begin{align*}
    \E\|s_{t_k}(X_t)-s_t(X_t)\|^2
    &\lesssim L^2_sd\tau+L^2_s\tau^2\,\E\|X_t\|^2+(1+L_s)^2\tau^2\,\E\|s_t(X_t)\|^2\\
    &\lesssim L^2_sd\tau+L^2_s(1+L_s)^2\mathcal M\tau^2 .
\end{align*}
Inserting both bounds into \eqref{eq:mismatch-split} completes the proof.
\end{proof}

\subsection{Proof of Theorem~\ref{thm:overdamped-stationarity}}
\label{app:od-final}

\begin{restated}[Averaged Fisher stationarity, overdamped]{Theorem}{\ref{thm:overdamped-stationarity}}
Suppose Assumptions~\ref{ass:lipschitz}--\ref{ass:expmoment} hold and $H_0<\infty$. There are universal constants $c$ such that if
$h\le c/(1+L_s)$ and $T\big[L^2_sdh+L^2_s(1+L_s)^2\mathcal Mh^2\big]\le cd$,
then $\sup_{0\le t\le T}\KL(\hp_t\,\|\,q_t)<H_0+d$ and
\[
    \frac1T\int_0^T\F(\hp_t\,\|\,q_t)\dd t
    \;\lesssim\;
    \frac{H_0}T
    +L^2_sdh
    +L^2_s(1+L_s)^2\mathcal Mh^2.
\]
\end{restated}

\begin{proof}
Write $\mathrm H(t):=\KL(\hp_t\,\|\,q_t)$, a continuous function of $t$ under the standing conventions. Whenever $\mathrm H(s)\le H_0+d$ for all $s\le t$, Proposition~\ref{prop:master} and Lemma~\ref{lem:local-mismatch} give, for $t\in[t_k,t_{k+1}]$ and $\tau=t-t_k$,
\begin{equation}
    \mathrm H'(t)\le-\frac12\F(\hp_t\,\|\,q_t)+C\big[L^2_sd\tau+L^2_s(1+L_s)^2\mathcal M\tau^2\big] .
    \label{eq:Hprime}
\end{equation}
The source term integrates to $C[L^2_sdh^2/2+L^2_s(1+L_s)^2\mathcal Mh^3/3]$ over a full step, to less over a partial one, and there are $T/h$ steps. So, for every $t\le T$ up to which \eqref{eq:Hprime} is in force,
\begin{equation}
    \mathrm H(t)+\frac12\int_0^t\F(\hp_s\,\|\,q_s)\dd s
    \le H_0+CT\big[L^2_sdh+L^2_s(1+L_s)^2\mathcal Mh^2\big]
    \le H_0+\frac d2 ,
    \label{eq:H-bootstrap}
\end{equation}
where the last inequality is the second condition of the theorem, once $c\le1/2C$.

The bootstrap closes by continuity. Let $t_\star:=\sup\{t\in[0,T]:\mathrm H(s)\le H_0+d\text{ for all }s\le t\}$; since $\mathrm H(0)=H_0$, $t_\star>0$ and $\mathrm H\le H_0+d$ on $[0,t_\star]$, so \eqref{eq:H-bootstrap} applies up to $t_\star$ and gives $\mathrm H(t_\star)\le H_0+d/2$. Were $t_\star<T$, $\mathrm H$ would stay below $H_0+d$ slightly beyond $t_\star$ by continuity, contradicting its definition. Hence $t_\star=T$, \eqref{eq:Hprime} holds on all of $[0,T]$. The inequality \eqref{eq:H-bootstrap} at $t=T$ implies
\[
    \frac12\int_0^T\F(\hp_t\,\|\,q_t)\dd t
    \le H_0+CT\big[L^2_sdh+L^2_s(1+L_s)^2\mathcal Mh^2\big] .
\]
Dividing by $T$ proves \eqref{eq:main-fisher}.
\end{proof}

\section{Proofs for Section~\texorpdfstring{\ref{sec:underdamped-stat}}{5}: the underdamped sampler}
\label{app:ud-proofs}

Throughout this appendix, $\hp_i$ and $\hq_i$ are the grid laws of the numerical chain started from $\hp_0$ and from the exact initialization $\hq_0=\rho_T$ respectively, $\hp_{i,\tau}$ and $\hq_{i,\tau}$ are their within-step interpolations \eqref{eq:interpolation}, and $C$, $E$, $\Pi_v$, $D$ are the matrices of \eqref{eq:matrices}. Assumption~\ref{ass:classical} supplies the positive densities and the finite $H_0$, $F_0$ of the standing conventions; since $\hq_0=q_0$, $H_0=\KL(\hp_0\,\|\,\hq_0)$ and $F_0=\F(\hp_0\,\|\,\hq_0)$. 

The argument follows the two-factor form \eqref{eq:factorization-main} of the integrator. The kick is a deterministic map applied to both laws, so it preserves relative entropy and transforms the relative score by the inverse transpose of its Jacobian (Lemma~\ref{lem:common-map}). The score-free leg is a linear diffusion, along which relative entropy dissipates the velocity part of the relative score, and a twisted Fisher divergence with a time-dependent metric can dissipate the full relative score, at the price of a term involving the mixed Hessian of the reference law, which Assumption~\ref{ass:reference} controls (Proposition~\ref{prop:ou-identities}). Appendix~\ref{app:ud-strategy} combines the two into a Lyapunov functional, lists the four properties (M1)--(M4) that the moving metric must have, and shows that they make the functional decrease at least as fast as $\F(\hp_{i,\tau}\,\|\,\hq_{i,\tau})$ along each leg and match at the grid points, so that the bound telescopes (Proposition~\ref{prop:one-step}). Appendix~\ref{app:ud-metric} motivates and defines the metric, verifies (M1)--(M3), and completes the proof of the theorem; the remaining property (M4), a $2\times2$ matrix inequality, is verified in Appendix~\ref{app:ud-certificate}.

\subsection{Reverse dynamics and the integrator}
\label{app:ud-integrator}

\begin{lemma}[Exact two-factor decomposition of the integrator]
\label{lem:factorization}
Let $\kernel_i$ be the transition kernel of the update \eqref{eq:update}. Then for every input law $\mu$,
\begin{equation}
    \mu\kernel_i=\mathcal R_{h_i}\big((\Psi_{i,h_i})_\#\mu\big),
    \label{eq:factorization}
\end{equation}
with the kick $\Psi_{i,h}$ of \eqref{eq:B-Phi} and the score-free semigroup $\mathcal R_\tau$ of \eqref{eq:score-free}.
\end{lemma}

\begin{proof}
Fix $z$ and let $f:=4Es_i(z)$, a constant vector. The frozen-score step from $z$ is the time-$h$ value of the linear SDE $\dd X=(CX+f)\dd t+2E\,\dd\bar B$, $X_0=z$, which by variation of constants is
\[
    X_h=e^{hC}\Big(z+\int_0^he^{-rC}f\dd r\Big)+\int_0^he^{(h-u)C}2E\,\dd\bar B_u
    =e^{hC}\big(z+B_hs_i(z)\big)+\zeta,
\]
where $\zeta\sim\mathcal N\big(0,\int_0^he^{vC}De^{vC^\top}\dd v\big)=\mathcal N(0,Q_h)$. This is $\mathcal R_h$ applied to the kicked point $\Psi_{i,h}(z)$.
\end{proof}

\subsection{Two laws under a common map}
\label{app:ud-two-density}

\begin{lemma}[Common deterministic pushforward]
\label{lem:common-map}
Let $\mathsf T\colon\R^{2d}\to\R^{2d}$ be a $C^1$ diffeomorphism, and let $\mu^+=\mathsf T_\#\mu$, $\nu^+=\mathsf T_\#\nu$ for positive densities $\mu,\nu$. With $r=\log(\mu/\nu)$ and $r^+=\log(\mu^+/\nu^+)$,
\begin{equation}
    r^+\circ\mathsf T=r,
    \qquad
    \KL(\mu^+\,\|\,\nu^+)=\KL(\mu\,\|\,\nu),
    \qquad
    (\nabla r^+)\circ\mathsf T=(\nabla\mathsf T)^{-\top}\nabla r .
    \label{eq:common-map-identities}
\end{equation}
Consequently, for every symmetric matrix $M$,
\begin{equation}
    \F_M(\mu^+\,\|\,\nu^+)
    =\int\mu\,\Big\langle\nabla r,\,(\nabla\mathsf T)^{-1}M(\nabla\mathsf T)^{-\top}\nabla r\Big\rangle .
    \label{eq:twist-pullback}
\end{equation}
\end{lemma}

\begin{proof}
By the change of variables formula, $\mu^+\circ\mathsf T=\mu/|\det\nabla\mathsf T|$ and likewise for $\nu$; the Jacobians cancel in the ratio, which is the first identity. The second is obtained by integrating $r^+$. Differentiating the first by the chain rule gives the third, and inserting the third into $\F_M(\mu^+\,\|\,\nu^+)=\int\mu\,\ip{(\nabla r^+)\circ\mathsf T}{M\,(\nabla r^+)\circ\mathsf T}$ implies \eqref{eq:twist-pullback}.
\end{proof}

\subsection{Identities on the score-free leg}
\label{app:ud-ou-leg}

On the score-free leg both interpolated densities in \eqref{eq:interpolation} solve
\begin{equation}
    \partial_\tau\nu_\tau=-\nabla\cdot(\nu_\tau\,Cz)+2\Delta_v\nu_\tau ,
    \label{eq:ou-pde}
\end{equation}
the Fokker--Planck equation of \eqref{eq:score-free}. For a pair $(\mu_\tau,\nu_\tau)$ of solutions put
\[
    r_\tau=\log\frac{\mu_\tau}{\nu_\tau},
    \qquad
    g_\tau=\nabla r_\tau,
    \qquad
    g_{v,\tau}=E^\top g_\tau=\nabla_vr_\tau ,
\]
and let $\nabla_z\nabla_vr_\tau$ and $\nabla_z\nabla_v\log\nu_\tau$ denote the $2d\times d$ matrices of mixed second derivatives, the latter being the mixed Hessian of Assumption~\ref{ass:reference}. Unlike the exact reverse flow of Appendix~\ref{app:decay-proofs}, whose drift contains the reference score, the score-free leg is oblivious to the pair it transports; the price is the last term of \eqref{eq:twisted-identity} below, in which the reference law enters through its mixed Hessian.

\begin{proposition}[Entropy and twisted-Fisher identities on the score-free leg]
\label{prop:ou-identities}
Let $(\mu_\tau)_\tau$ and $(\nu_\tau)_\tau$ be positive densities solving \eqref{eq:ou-pde}. Then
\begin{equation}
    \frac{\dd}{\dd\tau}\KL(\mu_\tau\,\|\,\nu_\tau)=-2\,\Fv(\mu_\tau\,\|\,\nu_\tau),
    \label{eq:entropy-identity}
\end{equation}
and for every $C^1$ family $\tau\mapsto M(\tau)$ of symmetric matrices,
\begin{equation}
    \begin{aligned}
    \frac{\dd}{\dd\tau}\F_{M(\tau)}(\mu_\tau\,\|\,\nu_\tau)
    ={}&\int\mu_\tau\,\Big\langle g_\tau,\big(\dot M(\tau)-CM(\tau)-M(\tau) C^\top\big)g_\tau\Big\rangle\\
    &-4\int\mu_\tau\operatorname{tr}\!\Big((\nabla_z\nabla_vr_\tau)^\top M(\tau)\,\nabla_z\nabla_vr_\tau\Big)\\
    &+8\int\mu_\tau\,\Big\langle M(\tau) g_\tau,\,(\nabla_z\nabla_v\log\nu_\tau)\,g_{v,\tau}\Big\rangle.
    \end{aligned}
    \label{eq:twisted-identity}
\end{equation}
\end{proposition}

\begin{proof}
Equation \eqref{eq:ou-pde} is the continuity equation with velocity field
\[
    u^\nu_\tau=Cz-2E\nabla_v\log\nu_\tau ,
    \qquad\text{so that}\qquad
    u^\mu_\tau-u^\nu_\tau=-2Eg_{v,\tau},
\]
and Lemma~\ref{lem:kl-identity} implies
\[
\frac{\dd}{\dd\tau}\KL(\mu_\tau\,\|\,\nu_\tau)=\int\mu_\tau\,\ip{g_\tau}{-2Eg_{v,\tau}}=-2\int\mu_\tau\|g_{v,\tau}\|^2 =-2\,\Fv(\mu_\tau\,\|\,\nu_\tau).
\]

For the second identity, divide \eqref{eq:ou-pde} by $\nu_\tau$. Using
\[
    \nabla\cdot(\nu_\tau Cz)=\nu_\tau\operatorname{tr}C+Cz\cdot\nabla\nu_\tau,
    \qquad
    \frac{\Delta_v\nu_\tau}{\nu_\tau}=\Delta_v\log\nu_\tau+\|\nabla_v\log\nu_\tau\|^2,
\]
we get
\[
    \partial_\tau\log\nu_\tau=-\operatorname{tr}C-Cz\cdot\nabla\log\nu_\tau+2\Delta_v\log\nu_\tau+2\|\nabla_v\log\nu_\tau\|^2 .
\]
Subtracting this from the same equation for $\mu_\tau$, and writing $\|\nabla_v\log\mu_\tau\|^2-\|\nabla_v\log\nu_\tau\|^2=\|g_{v,\tau}\|^2+2\ip{g_{v,\tau}}{\nabla_v\log\nu_\tau}$,
\begin{equation}
    \partial_\tau r_\tau=-Cz\cdot\nabla r_\tau+2\Delta_vr_\tau+2\|g_{v,\tau}\|^2+4\ip{g_{v,\tau}}{\nabla_v\log\nu_\tau}.
    \label{eq:ratio-pde}
\end{equation}
Let $S_\tau=\nabla^2r_\tau$, so that $\nabla_z\nabla_vr_\tau=S_\tau E$, and write $\mathcal H_\tau=\nabla_z\nabla_v\log\nu_\tau$. Taking the gradient of \eqref{eq:ratio-pde} with
\[
    \nabla(Cz\cdot\nabla r_\tau)=C^\top g_\tau+S_\tau Cz,
    \qquad
    \nabla\|g_{v,\tau}\|^2=2S_\tau Eg_{v,\tau},
\]
\[
    \nabla\ip{g_{v,\tau}}{\nabla_v\log\nu_\tau}=S_\tau E\nabla_v\log\nu_\tau+\mathcal H_\tau g_{v,\tau},
\]
we obtain
\begin{equation}
    \begin{aligned}
    \partial_\tau g_\tau
    &=-C^\top g_\tau-S_\tau Cz+2\Delta_vg_\tau+4S_\tau E\big(g_{v,\tau}+\nabla_v\log\nu_\tau\big)+4\mathcal H_\tau g_{v,\tau}\\
    &=-C^\top g_\tau-S_\tau Cz+2\Delta_vg_\tau+4S_\tau E\nabla_v\log\mu_\tau+4\mathcal H_\tau g_{v,\tau} .
    \end{aligned}
    \label{eq:score-pde}
\end{equation}

Now differentiate $\F_{M(\tau)}(\mu_\tau\,\|\,\nu_\tau)=\int\mu_\tau\,\ip{g_\tau}{M(\tau) g_\tau}$. Write $M=M(\tau)$ for brevity. Using the continuity equation of $\mu_\tau$ for the first term and $\nabla\ip{g_\tau}{Mg_\tau}=2S_\tau Mg_\tau$,
\[
    \frac{\dd}{\dd\tau}\F_{M(\tau)}(\mu_\tau\,\|\,\nu_\tau)
    =2\int\mu_\tau\,\ip{S_\tau Mg_\tau}{u^\mu_\tau}
    +\int\mu_\tau\,\ip{g_\tau}{\dot M g_\tau}
    +2\int\mu_\tau\,\ip{M g_\tau}{\partial_\tau g_\tau}.
\]
Since $u^\mu_\tau=Cz-2E\nabla_v\log\mu_\tau$ and $S_\tau$, $M$ are symmetric, the first term equals
\[
    2\int\mu_\tau\,\ip{Mg_\tau}{S_\tau Cz}-4\int\mu_\tau\,\ip{Mg_\tau}{S_\tau E\nabla_v\log\mu_\tau}.
\]
In the third term we insert \eqref{eq:score-pde} and integrate the Laplacian by parts: as $\partial_{v_j}g_\tau=S_\tau e_{d+j}$ and $\partial_{v_j}(\mu_\tau\,M g_\tau)=\mu_\tau\,\partial_{v_j}\!\log\mu_\tau\;M g_\tau+\mu_\tau\,M S_\tau e_{d+j}$,
\[
    4\int\mu_\tau\,\ip{M g_\tau}{\Delta_vg_\tau}
    =-4\int\mu_\tau\,\ip{M g_\tau}{S_\tau E\nabla_v\log\mu_\tau}-4\int\mu_\tau\operatorname{tr}\!\big((S_\tau E)^\top M S_\tau E\big).
\]
Collecting everything,
\begin{align*}
    \frac{\dd}{\dd\tau}\F_{M(\tau)}(\mu_\tau\,\|\,\nu_\tau)
    ={}&\int\mu_\tau\,\ip{g_\tau}{\dot M g_\tau}
    -2\int\mu_\tau\,\ip{M g_\tau}{C^\top g_\tau}
    +(2-2)\int\mu_\tau\,\ip{Mg_\tau}{S_\tau Cz}\\
    &+(-4+8-4)\int\mu_\tau\,\ip{Mg_\tau}{S_\tau E\nabla_v\log\mu_\tau}\\
    &-4\int\mu_\tau\operatorname{tr}\!\big((S_\tau E)^\top M S_\tau E\big)
    +8\int\mu_\tau\,\ip{M g_\tau}{\mathcal H_\tau g_{v,\tau}} .
\end{align*}
Since $2\ip{Mg_\tau}{C^\top g_\tau}=\ip{g_\tau}{(CM+MC^\top)g_\tau}$, this is \eqref{eq:twisted-identity}.
\end{proof}

\subsection{A Lyapunov functional with a moving metric}
\label{app:ud-strategy}

\paragraph{Motivation.}
By \eqref{eq:entropy-identity}, relative entropy dissipates only the velocity part $\Fv$ of the relative score. To reach the full $\F$ we add, as in the hypocoercive functional of Section~\ref{sec:decay}, a twisted Fisher divergence $\F_M$. Identity \eqref{eq:twisted-identity} shows what each ingredient contributes along a leg: the \emph{drift term} $\ip g{(\dot M-CM-MC^\top)g}$ is the only one that can dissipate in every direction, including position; the \emph{trace term} has the right sign as soon as $M\succeq0$; and the \emph{mixed-Hessian term} has no sign, but by Assumption~\ref{ass:reference} it costs only a multiple of $\Fv$, which the entropy pays.

A constant metric does not survive the kick. By \eqref{eq:twist-pullback}, the kick replaces the metric $M$ by $(\nabla\Psi)^{-1}M(\nabla\Psi)^{-\top}$, which can be larger than $M$ in some directions. We therefore let the metric move within each leg: it starts at a matrix $M(0)$ slightly below a fixed grid metric $\Theta$, low enough that the increase by the kick is absorbed by the decrease in the metric $\Theta-M(0)$, and returns to $\Theta$ at the end of the leg, so that consecutive legs match at the grid points. The deficit $\Theta-M(0)$ is a budget that the dissipation along the leg must pay back.

\paragraph{Requirements.}
Fix a symmetric positive definite matrix $\Theta$ (the grid metric, chosen in \eqref{eq:S-P}) and a weight $a>0$. Fix a step $i$ and write $h=h_i$ and $\Psi=\Psi_{i,h}$. A $C^1$ family $\{M(\tau)\}_{0\le\tau\le h}$ of symmetric $2d\times2d$ matrices is an \emph{admissible metric} for step $i$ if the following hold.
\begin{enumerate}
    \item[(M1)] \emph{Kick compatibility:} for every $z\in\R^{2d}$,
    \[
        M(0)\preceq\nabla\Psi(z)\,\Theta\,\nabla\Psi(z)^\top .
    \]
    \item[(M2)] \emph{Matching at the grid:} $M(h)=\Theta$.
    \item[(M3)] \emph{Positivity:} $M(\tau)\succeq0$ for $0\le\tau\le h$.
    \item[(M4)] \emph{Dissipation:} for $0\le\tau\le h$,
    \[
        -\big(\dot M(\tau)-CM(\tau)-M(\tau)C^\top\big)-\frac14M(\tau)^2+(2a-64\Lambda^2)\Pi_v\succeq\Itwo .
    \]
\end{enumerate}
Given an admissible metric for step $i$, define on the $i$th leg and at the grid points
\begin{equation}
 \begin{aligned}
 \mathcal Y_i(\tau)&:=
 a\KL(\hp_{i,\tau}\,\|\,\hq_{i,\tau})
 +\F_{M(\tau)}(\hp_{i,\tau}\,\|\,\hq_{i,\tau}),
 \qquad0\le\tau\le h_i,\\
 \mathcal Y_i^{\mathrm{grid}}&:=a\KL(\hp_i\,\|\,\hq_i)+\F_\Theta(\hp_i\,\|\,\hq_i).
 \end{aligned}
 \label{eq:leg-Lyapunov}
\end{equation}

(M1) and (M2) make $\mathcal Y_i$ start below $\mathcal Y_i^{\mathrm{grid}}$ and end at $\mathcal Y_{i+1}^{\mathrm{grid}}$, while (M3) and (M4) make it decrease at least as fast as $\F$ in between.

\begin{proposition}[One-step dissipation]
\label{prop:one-step}
Suppose $\Psi_{i,h_i}$ is a $C^1$ diffeomorphism and $\{M(\tau)\}$ is an admissible metric for step $i$. Then, for $0<\tau<h_i$,
\begin{equation}
 \mathcal Y_i'(\tau)\le-\F(\hp_{i,\tau}\,\|\,\hq_{i,\tau}),
 \label{eq:leg-dissipation}
\end{equation}
and consequently
\begin{equation}
 \mathcal Y_{i+1}^{\mathrm{grid}}
 +\int_0^{h_i}\F(\hp_{i,\tau}\,\|\,\hq_{i,\tau})\dd\tau
 \le\mathcal Y_i^{\mathrm{grid}}.
 \label{eq:grid-dissipation}
\end{equation}
If this holds for every $0\le i<N$ with the same $a$ and $\Theta$, then
\begin{equation}
 \sum_{i=0}^{N-1}\int_0^{h_i}\F(\hp_{i,\tau}\,\|\,\hq_{i,\tau})\dd\tau
 +\mathcal Y_N^{\mathrm{grid}}
 \le\mathcal Y_0^{\mathrm{grid}}
 =aH_0+\F_\Theta(\hp_0\,\|\,\hq_0) .
 \label{eq:full-telescope}
\end{equation}
\end{proposition}

\begin{proof}
Write $h=h_i$, $M=M(\tau)$, and apply Proposition~\ref{prop:ou-identities} to $\mu_\tau=\hp_{i,\tau}$ and $\nu_\tau=\hq_{i,\tau}$, with
\[
    g_\tau=\nabla\log\frac{\hp_{i,\tau}}{\hq_{i,\tau}},
    \qquad
    g_{v,\tau}=E^\top g_\tau,
    \qquad
    \mathcal H_\tau=\nabla_z\nabla_v\log\hq_{i,\tau}.
\]
By Assumption~\ref{ass:reference}, $\|\mathcal H_\tau\|_\op\le\Lambda$, so the mixed-Hessian term of \eqref{eq:twisted-identity} is bounded pointwise by
\[
 8\ip{Mg_\tau}{\mathcal H_\tau g_{v,\tau}}
 \le8\Lambda\|Mg_\tau\|\,\|g_{v,\tau}\|
 \le\frac14\|Mg_\tau\|^2+64\Lambda^2\|g_{v,\tau}\|^2 .
\]
The trace term of \eqref{eq:twisted-identity} is nonpositive by (M3). Adding $a$ times \eqref{eq:entropy-identity}, which contributes $-2a\int\hp_{i,\tau}\|g_{v,\tau}\|^2$, and writing $\|Mg_\tau\|^2=\ip{g_\tau}{M^2g_\tau}$ and $\|g_{v,\tau}\|^2=\ip{g_\tau}{\Pi_vg_\tau}$, we get
\[
 \mathcal Y_i'(\tau)\le\int\hp_{i,\tau}\,\Big\langle g_\tau,\Big(\dot M-CM-MC^\top+\frac14M^2-(2a-64\Lambda^2)\Pi_v\Big)g_\tau\Big\rangle\dd z
 \le-\int\hp_{i,\tau}\|g_\tau\|^2
\]
by (M4). This is \eqref{eq:leg-dissipation}.

At the start of the leg, $\hp_{i,0}=(\Psi_{i,h})_\#\hp_i$ and $\hq_{i,0}=(\Psi_{i,h})_\#\hq_i$. Relative entropy is invariant (Lemma~\ref{lem:common-map}), and (M1) is equivalent to $(\nabla\Psi_{i,h})^{-1}M(0)(\nabla\Psi_{i,h})^{-\top}\preceq\Theta$, so \eqref{eq:twist-pullback} shows
\[
 \F_{M(0)}(\hp_{i,0}\,\|\,\hq_{i,0})\le\F_\Theta(\hp_i\,\|\,\hq_i),
 \qquad\text{hence}\qquad
 \mathcal Y_i(0)\le\mathcal Y_i^{\mathrm{grid}}.
\]
At the end of the leg, $M(h)=\Theta$ by (M2) and $\hp_{i,h}=\hp_{i+1}$, $\hq_{i,h}=\hq_{i+1}$, so $\mathcal Y_i(h)=\mathcal Y_{i+1}^{\mathrm{grid}}$. Integrating \eqref{eq:leg-dissipation} over $[0,h]$ gives \eqref{eq:grid-dissipation}, and summing \eqref{eq:grid-dissipation} over $i=0,\dots,N-1$ gives \eqref{eq:full-telescope}, since $\hq_0=q_0$.
\end{proof}

Theorem~\ref{thm:underdamped-stationarity} is thus reduced to constructing, for every step, an admissible metric with $a\lesssim\Lambda^2$ and a grid metric $\Theta$ comparable to $\Itwo$. This is done in Appendix~\ref{app:ud-metric}, with $a=88\Lambda^2$ and $\Theta\preceq4\Itwo$.

\begin{remark}[The velocity part alone]
\label{rem:velocity-identity}
Relative entropy is invariant under the common kick (a diffeomorphism by Lemma~\ref{lem:combined-map}) and satisfies \eqref{eq:entropy-identity} on every leg, so without any metric
\begin{equation}
 2\sum_{i=0}^{N-1}\int_0^{h_i}\Fv(\hp_{i,\tau}\,\|\,\hq_{i,\tau})\dd\tau
 =H_0-\KL(\hp_N\,\|\,\hq_N)\le H_0 ,
 \label{eq:velocity-stationarity}
\end{equation}
the discrete counterpart of \eqref{eq:integrated-dissipation-ud}. This identity does not use Assumption~\ref{ass:reference}; the moving metric is what upgrades $\Fv$ to $\F$.
\end{remark}

\subsection{Construction of the moving metric}
\label{app:ud-metric}

All matrices in this subsection and the next have scalar $d\times d$ blocks, $X=\bar X\otimes\Id$ with $\bar X\in\R^{2\times2}$. We identify $X$ with its $2\times2$ representative $\bar X$, write $X_{xx}$, $X_{xv}$, $X_{vv}$ for its entries, and note that $X\succeq Y$ if and only if $\bar X\succeq\bar Y$. Fix one step, of length $h$, drop the index $i$, and write $\Psi_h=\Psi_{i,h}$ and $s=s_i$.

\paragraph{Definition.}
Set
\begin{equation}
 \Theta:=\begin{pmatrix}\frac12&-1\\-1&3\end{pmatrix},
 \qquad A:=C\Theta+\Theta C^\top
 =\begin{pmatrix}2&-\frac92\\-\frac92&10\end{pmatrix},
 \qquad
 \Theta^{-1}=\begin{pmatrix}6&2\\2&1\end{pmatrix},
 \label{eq:S-P}
\end{equation}
so that $0\prec\Theta\preceq4\Itwo$, and
\begin{equation}
 \begin{aligned}
 P_h&:=\frac h3\Theta^2+\frac{3\Lambda^2}{h}B_hB_h^\top,\\
 M_h(\tau)&:=\Big(1-\frac\tau h\Big)e^{\tau C}(\Theta-P_h)e^{\tau C^\top}+\frac\tau h\,e^{(\tau-h)C}\Theta e^{(\tau-h)C^\top},
 \qquad0\le\tau\le h .
 \end{aligned}
 \label{eq:moving-metric}
\end{equation}
Thus
\begin{equation}
 M_h(0)=\Theta-P_h,\qquad M_h(h)=\Theta ,
 \label{eq:metric-endpoints}
\end{equation}
and, since $e^{-\tau C}M_h(\tau)e^{-\tau C^\top}$ is affine in $\tau$ with derivative $\frac1h\big(e^{-hC}\Theta e^{-hC^\top}-\Theta+P_h\big)$,
\begin{equation}
 \dot M_h(\tau)-CM_h(\tau)-M_h(\tau)C^\top
 =\frac1h\,e^{\tau C}\big(e^{-hC}\Theta e^{-hC^\top}-\Theta+P_h\big)e^{\tau C^\top}.
 \label{eq:drift-removed}
\end{equation}

\paragraph{Verification.}
Requirement (M2) is \eqref{eq:metric-endpoints}. Requirements (M1) and (M3) are checked in Lemma~\ref{lem:combined-map}, and (M4) in Lemma~\ref{lem:moving-metric}.

\begin{lemma}[The kick and the budget $P_h$: (M1) and (M3)]
\label{lem:combined-map}
If $h\le(64\Lambda^2)^{-1}$, then $\Psi_h$ is a global $C^1$ diffeomorphism,
\[
    \Theta-P_h\succeq\frac16\Theta,
    \qquad
    M_h(\tau)\succ0\quad(0\le\tau\le h),
\]
and, for every $z\in\R^{2d}$,
\begin{equation}
 M_h(0)\preceq\nabla\Psi_h(z)\,\Theta\,\nabla\Psi_h(z)^\top .
 \label{eq:combined-domination}
\end{equation}
Consequently, for all positive densities $\mu,\nu$,
\begin{equation}
 \F_{M_h(0)}\big((\Psi_h)_\#\mu\,\big\|\,(\Psi_h)_\#\nu\big)
 \le\F_\Theta(\mu\,\|\,\nu).
 \label{eq:combined-functional}
\end{equation}
\end{lemma}

\begin{proof}
\emph{Diffeomorphism.} By \eqref{eq:B-Phi}, $B_h=(\beta_x,\beta_v)^\top$ with
\[
    \beta_x=4\big[1-(1+h)e^{-h}\big]\in[0,2h^2],
    \qquad
    \beta_v=4he^{-h}\in[0,4h],
\]
where we used $1-(1+h)e^{-h}\le h^2/2$. Hence $\|B_h\|_\op\le4h\sqrt{1+h^2/4}$. As $h\le(64\Lambda^2)^{-1}$ and $\Lambda\ge1$, we have $\Lambda h\le\frac1{64}$, so $\Lambda\|B_h\|_\op<1$. Thus $z\mapsto y-B_hs(z)$ is a contraction for every $y$, so $\Psi_h(z)=z+B_hs(z)$ is a bijection; its Jacobian $\Itwo+B_h\nabla s(z)$ is invertible, and by the inverse function theorem $\Psi_h$ is a $C^1$ diffeomorphism.

\emph{Kick compatibility.} Put $\Gamma=B_h\nabla s(z)$, so that $\nabla\Psi_h=\Itwo+\Gamma$ and
\[
    (\Itwo+\Gamma)\Theta(\Itwo+\Gamma)^\top=\Theta+\Gamma\Theta+\Theta\Gamma^\top+\Gamma\Theta\Gamma^\top .
\]
The last term is positive semidefinite, and for every $g$, using $\|\Gamma^\top g\|=\|(\nabla s)^\top B_h^\top g\|\le\Lambda\|B_h^\top g\|$,
\[
 \big|\ip g{(\Gamma\Theta+\Theta\Gamma^\top)g}\big|
 =2\big|\ip{\Theta g}{\Gamma^\top g}\big|
 \le\frac h3\|\Theta g\|^2+\frac3h\|\Gamma^\top g\|^2
 \le\ip g{P_hg}.
\]
Hence $(\Itwo+\Gamma)\Theta(\Itwo+\Gamma)^\top\succeq\Theta-P_h=M_h(0)$, which is \eqref{eq:combined-domination}. It is equivalent to $(\nabla\Psi_h)^{-1}M_h(0)(\nabla\Psi_h)^{-\top}\preceq\Theta$, and \eqref{eq:combined-functional} follows from \eqref{eq:twist-pullback}.

\emph{Positivity.} For $\Theta-P_h\succeq\frac16\Theta$ it suffices that $\|\Theta^{-1/2}P_h\Theta^{-1/2}\|_\op\le\frac56$. Using $\|\Theta\|_\op\le4$, $\Theta^{-1}$ from \eqref{eq:S-P}, and the bounds on $\beta_x,\beta_v$,
\begin{align*}
 \|\Theta^{-1/2}P_h\Theta^{-1/2}\|_{\op}
 &\le\frac h3\|\Theta\|_{\op}+\frac{3\Lambda^2}{h}B_h^\top\Theta^{-1}B_h
 =\frac h3\|\Theta\|_{\op}+\frac{3\Lambda^2}{h}\big(6\beta_x^2+4\beta_x\beta_v+\beta_v^2\big)\\
 &\le\frac{4h}3+48\Lambda^2h\Big(1+2h+\frac32h^2\Big)
 \le\frac1{48}+\frac34\Big(1+\frac1{32}+\frac3{8192}\Big)
 <\frac56 .
\end{align*}
Finally, $M_h(\tau)$ is a convex combination of
\[
    e^{\tau C}(\Theta-P_h)e^{\tau C^\top}\succeq\frac16e^{\tau C}\Theta e^{\tau C^\top}\succ0
    \qquad\text{and}\qquad
    e^{(\tau-h)C}\Theta e^{(\tau-h)C^\top}\succ0 ,
\]
so $M_h(\tau)\succ0$.
\end{proof}

\begin{lemma}[Dissipation certificate: (M4)]
\label{lem:moving-metric}
For $h\le(64\Lambda^2)^{-1}$ and $0\le\tau\le h$,
\begin{equation}
 -\big(\dot M_h(\tau)-CM_h(\tau)-M_h(\tau)C^\top\big)
 -\frac14M_h(\tau)^2+112\Lambda^2\Pi_v\succeq\Itwo .
 \label{eq:interaction-bound}
\end{equation}
\end{lemma}

\paragraph{Proof of Theorem~\ref{thm:underdamped-stationarity}.}
We now collect the pieces.

\begin{restated}[Averaged relative-Fisher stationarity, underdamped]{Theorem}{\ref{thm:underdamped-stationarity}}
Suppose Assumptions~\ref{ass:score-map}--\ref{ass:classical} hold and $\max_{0\le i<N}h_i\le(64\Lambda^2)^{-1}$. Then
\[
    \frac{1}{T}\sum_{i=0}^{N-1}\int_0^{h_i}\F(\hp_{i,\tau}\,\|\,\hq_{i,\tau})\dd\tau
    \;\lesssim\; \frac{\,\Lambda^2H_0+F_0}{T} .
\]
\end{restated}

\begin{proof}
Fix $i$. By Lemma~\ref{lem:combined-map}, $\Psi_{i,h_i}$ is a $C^1$ diffeomorphism and $M_{h_i}$ satisfies (M1) and (M3); (M2) is \eqref{eq:metric-endpoints}; and by Lemma~\ref{lem:moving-metric}, $M_{h_i}$ satisfies (M4) with $2a-64\Lambda^2=112\Lambda^2$, that is, with $a=88\Lambda^2$. Proposition~\ref{prop:one-step} therefore applies to every step, and \eqref{eq:full-telescope} together with $\Theta\preceq4\Itwo$ gives
\[
 \sum_{i=0}^{N-1}\int_0^{h_i}\F(\hp_{i,\tau}\,\|\,\hq_{i,\tau})\dd\tau
 +\mathcal Y_N^{\mathrm{grid}}
 \le88\Lambda^2H_0+\F_\Theta(\hp_0\,\|\,\hq_0)
 \le88\Lambda^2H_0+4F_0 .
\]
Since $\mathcal Y_N^{\mathrm{grid}}\ge0$ and $\sum_ih_i=T$,
\begin{equation}
 \frac1{T}\sum_{i=0}^{N-1}\int_0^{h_i}
       \F(\hp_{i,\tau}\,\|\,\hq_{i,\tau})\dd\tau
 \le\frac{88\Lambda^2H_0+4F_0}{T}.
 \label{eq:ud-refined-conclusion}
 \qedhere
\end{equation}
\end{proof}

\subsection{Proof of the dissipation certificate}
\label{app:ud-certificate}

\begin{restated}[Dissipation certificate: (M4)]{Lemma}{\ref{lem:moving-metric}}
For $h\le(64\Lambda^2)^{-1}$ and $0\le\tau\le h$,
\[
    -\big(\dot M_h(\tau)-CM_h(\tau)-M_h(\tau)C^\top\big)-\frac14M_h(\tau)^2+112\Lambda^2\Pi_v\succeq\Itwo .
\]
\end{restated}

\begin{proof}
Throughout, $h\le(64\Lambda^2)^{-1}$ with $\Lambda\ge1$, so that $h\le\frac1{64}$,
and $0\le\tau\le h$, $u:=\tau/h\in[0,1]$, $R_t:=e^{tC}$. We work with $2\times2$ representatives. The proof has five steps: an exact formula for the left side of \eqref{eq:interaction-bound} (Step~1), elementary bounds on its ingredients (Steps~2--4), and the estimate of its three entries (Step~5).

\paragraph{Step 1: the matrix to be bounded.}
For a symmetric matrix $X$ write $X_t:=R_tXR_t^\top$; then $\frac{\dd}{\dd t}X_t=R_t(CX+XC^\top)R_t^\top$, and in particular $\frac{\dd}{\dd t}\Theta_t=A_t$. Put
\[
    K_t:=R_t\Theta^2R_t^\top,
    \qquad
    b=\binom{b_x}{b_v}:=R_\tau B_h,
    \qquad
    \bar A_\tau:=\frac1h\int_{\tau-h}^{\tau}A_t\dd t=\frac1h\big(\Theta_\tau-\Theta_{\tau-h}\big).
\]
By \eqref{eq:drift-removed} and the definition of $P_h$,
\[
 -\big(\dot M_h(\tau)-CM_h(\tau)-M_h(\tau)C^\top\big)
 =\frac1h\big(\Theta_\tau-\Theta_{\tau-h}\big)-\frac1h\,R_\tau P_hR_\tau^\top
 =\bar A_\tau-\frac13K_\tau-\frac{3\Lambda^2}{h^2}bb^\top .
\]
The left side of \eqref{eq:interaction-bound} is therefore
\begin{equation}
 \mathcal D_h(\tau)
 :=\bar A_\tau
  -\frac13K_\tau-\frac{3\Lambda^2}{h^2}bb^\top
  -\frac14M_h(\tau)^2+112\Lambda^2\Pi_v .
 \label{eq:ud-certificate-matrix}
\end{equation}
We shall prove
\begin{equation}
 \big(\mathcal D_h(\tau)\big)_{xx}\ge\frac{19}{16},\qquad
 \big|\big(\mathcal D_h(\tau)\big)_{xv}\big|\le\frac{17}5,\qquad
 \big(\mathcal D_h(\tau)\big)_{vv}\ge64\Lambda^2 .
 \label{eq:ud-certificate-entries}
\end{equation}
This yields \eqref{eq:interaction-bound}. Indeed, the diagonal entries of $\mathcal D_h(\tau)-I_2$ are at least $\frac3{16}$ and $64\Lambda^2-1\ge63$, so
\[
 \det\big(\mathcal D_h(\tau)-I_2\big)
 \ge\frac3{16}\cdot63-\Big(\frac{17}5\Big)^2
 =\frac{101}{400}>0 ,
\]
and a symmetric $2\times2$ matrix with positive diagonal and positive determinant is positive definite.

\paragraph{Step 2: scalar bounds on $\Theta_t$, $A_t$, $K_t$.}
By \eqref{eq:critical-exponentials}, $R_t=e^t\big(\begin{smallmatrix}1-t&-t\\t&1+t\end{smallmatrix}\big)$, and direct multiplication shows
{\setlength{\arraycolsep}{2.5pt}
\begin{align*}
 \Theta_t&=e^{2t}\begin{pmatrix}
 \frac12+t+\frac32t^2&-1-\frac52t-\frac32t^2\\
 -1-\frac52t-\frac32t^2&3+4t+\frac32t^2
 \end{pmatrix},
 \qquad
 A_t=e^{2t}\begin{pmatrix}
 2+5t+3t^2&-\frac92-8t-3t^2\\
 -\frac92-8t-3t^2&10+11t+3t^2
 \end{pmatrix},\\
 K_t&=e^{2t}\begin{pmatrix}
 \frac54+\frac92t+\frac{17}4t^2&-\frac72-\frac{35}4t-\frac{17}4t^2\\
 -\frac72-\frac{35}4t-\frac{17}4t^2&10+13t+\frac{17}4t^2
 \end{pmatrix}.
\end{align*}}
Every entry of these three matrices, and every trace, is $\pm f(t)$ with
\[
    f(t)=e^{2t}(\alpha+\beta t+\gamma t^2),
    \qquad
    \alpha\ge\tfrac12,\quad\beta\ge0,\quad0\le\gamma\le\tfrac{17}2 .
\]
For such $f$,
\[
    f'(t)=e^{2t}\big[(2\alpha+\beta)+2(\beta+\gamma)t+2\gamma t^2\big],
    \qquad
    f''(t)=e^{2t}\big[(4\alpha+4\beta+2\gamma)+4(\beta+2\gamma)t+4\gamma t^2\big],
\]
and both are positive on $|t|\le\frac1{64}$, because there the constant coefficient dominates the linear one. So $f$ is increasing and convex on $|t|\le\frac1{64}$, and lies above its tangent at $0$ and below its chords:
\begin{equation}
    f(0)+f'(0)\,t\;\le\;f(t)\;\le\;
    \begin{cases}
        f(0), & -\frac1{64}\le t\le0,\\[2pt]
        f(0)+f'(\tfrac1{64})\,t, & 0\le t\le\frac1{64},
    \end{cases}
    \qquad\text{and}\qquad f(t)\le f(\tfrac1{64}).
    \label{eq:tangent-chord}
\end{equation}
We apply \eqref{eq:tangent-chord} with the following constants, computed from the formulas above and $e^{1/32}<1.032$ (recall $\frac{\dd}{\dd t}\Theta_t=A_t$):
\begin{itemize}
    \item $f'(0)=9,\ 31,\ \frac92,\ 12,\ 17$ for $(A_t)_{xx}$, $(A_t)_{vv}$, $-(\Theta_t)_{xv}$, $\operatorname{tr}\Theta_t$, $-(A_t)_{xv}$;
    \item $f'(\frac1{64})\le\frac{11}5,\ \frac{24}5,\ \frac{21}2,\ 18,\ 8,\ 35$ for $(\Theta_t)_{xx}$, $-(\Theta_t)_{xv}$, $(\Theta_t)_{vv}$, $-(A_t)_{xv}$, $(K_t)_{xx}$, $(K_t)_{vv}$ 
    \item $f(\frac1{64})\le\frac{15}4,\ \frac{19}5,\ 12$ for $\operatorname{tr}\Theta_t$, $-(K_t)_{xv}$, $\operatorname{tr}K_t$ 
\end{itemize}
The result is, writing $t^+=\max\{t,0\}$: for $|t|\le\frac1{64}$,
\begin{equation}
 (A_t)_{xx}\ge2+9t,\qquad (A_t)_{vv}\ge10+31t,\qquad
 -\frac92-18\,t^+\le(A_t)_{xv}<0 ;
 \label{eq:ud-scalar-A}
\end{equation}
for $0\le t\le\frac1{64}$,
\begin{equation}
 \begin{gathered}
 0<(\Theta_t)_{xx}\le\frac12+\frac{11}5t,\qquad
 -1-\frac{24}5t\le(\Theta_t)_{xv}\le-1,\qquad
 (\Theta_t)_{vv}\le3+\frac{21}2t,\\
 \frac72\le\operatorname{tr}\Theta_t\le\frac{15}4,\qquad
 (K_t)_{xx}\le\frac54+8t,\qquad
 -\frac{19}5\le(K_t)_{xv}\le-\frac72,\\
 (K_t)_{vv}\le10+35t,\qquad
 \operatorname{tr}K_t\le12 ;
 \end{gathered}
 \label{eq:ud-scalar-plus}
\end{equation}
and for $-\frac1{64}\le t\le0$,
\begin{equation}
 \begin{gathered}
 0<(\Theta_t)_{xx}\le\frac12,\qquad
 -1\le(\Theta_t)_{xv}\le-1+\frac92|t|,\\
 (\Theta_t)_{vv}\le3,\qquad
 \frac72-12|t|\le\operatorname{tr}\Theta_t\le\frac72 .
 \end{gathered}
 \label{eq:ud-scalar-minus}
\end{equation}
In particular, for $-\frac1{64}\le t\le0$,
\[
    (\Theta_t^2)_{xx}=(\Theta_t)_{xx}^2+(\Theta_t)_{xv}^2\le\frac14+1=\frac54 .
\]

\paragraph{Step 3: the kick vector.}
Since $B_h=4\int_0^hR_{-r}E\dd r$ and $R_tE=e^t\binom{-t}{1+t}$,
\[
    b=R_\tau B_h=4\int_0^hR_{\tau-r}E\dd r=4\int_{\tau-h}^{\tau}e^t\binom{-t}{1+t}\dd t .
\]
The interval $[\tau-h,\tau]\subset[-h,h]$ has length $h$, so $\int_{\tau-h}^\tau|t|\dd t\le\frac{h^2}2$; and $0<e^t(1+t)\le(1+h)e^h$ there. Hence
\[
 |b_x|\le2h^2e^h,
 \qquad
 0<b_v\le4h(1+h)e^h .
\]
With $\Lambda^2h\le\frac1{64}$, $h\le\frac1{64}$ and $(1+h)^2e^{2h}\le e^{4h}\le1+4he^{1/16}$, this gives
\begin{equation}
 \begin{gathered}
 \frac{3\Lambda^2}{h^2}b_x^2\le12\Lambda^2h^2e^{2h}\le\frac h5,\qquad
 \frac{3\Lambda^2}{h^2}|b_xb_v|\le24\Lambda^2h(1+h)e^{2h}\le\frac25,\\
 \frac{3\Lambda^2}{h^2}b_v^2\le48\Lambda^2(1+h)^2e^{2h}\le48\Lambda^2+3e^{1/16}\le48\Lambda^2+\frac{16}5,\\
 \frac{3\Lambda^2}{h}\big(b_x^2+b_v^2\big)\le\frac3{16}h^2e^{2h}+\frac34e^{1/16}\le\frac45 .
 \end{gathered}
 \label{eq:ud-rank-bounds}
\end{equation}
Only the size of $b_x$ is used; it has no fixed sign.

\paragraph{Step 4: the moving metric.}
Put $W:=R_\tau(\Theta-P_h)R_\tau^\top$. By \eqref{eq:moving-metric},
\begin{equation}
 M_h(\tau)=(1-u)W+u\,\Theta_{\tau-h},\qquad
 W=\Theta_\tau-\frac h3K_\tau-\frac{3\Lambda^2}{h}bb^\top ,
 \label{eq:ud-W}
\end{equation}
with $\tau\in[0,h]$ and $\tau-h\in[-h,0]$.

\emph{Bounds on $W$.} By Lemma~\ref{lem:combined-map}, $W\succeq\frac16\Theta_\tau\succ0$, so $W_{xx},W_{vv}>0$. The two matrices subtracted from $\Theta_\tau$ in \eqref{eq:ud-W} are positive semidefinite, so the diagonal entries of $W$ are at most those of $\Theta_\tau$. For the off-diagonal entry, \eqref{eq:ud-scalar-plus} gives $-\frac h3(K_\tau)_{xv}\in\big[\frac76h,\frac{19}{15}h\big]$, and \eqref{eq:ud-rank-bounds} gives $\frac{3\Lambda^2}{h}|b_xb_v|\le\frac25h<\frac76h$. For the trace we use $\operatorname{tr}K_\tau\le12$ and the last bound of \eqref{eq:ud-rank-bounds}. Altogether,
\begin{equation}
 \begin{gathered}
 0<W_{xx}\le\frac12+\frac{11}5\tau,\qquad
 0<W_{vv}\le3+\frac{21}2\tau,\qquad
 -1-\frac{24}5\tau\le W_{xv}\le-1+\frac53h,\\
 \operatorname{tr}W\ge\frac72-4h-\frac45\ge\frac{21}8 .
 \end{gathered}
 \label{eq:ud-W-lower}
\end{equation}

\emph{Bounds on $M_h$.} Since $M_h(\tau)$ is a convex combination of $W$ and $\Theta_{\tau-h}$, \eqref{eq:ud-W-lower} and \eqref{eq:ud-scalar-minus} (at $t=\tau-h$, so $|t|\le h$) give
\begin{equation}
 \begin{gathered}
 -1-\frac{24}5h\le\big(M_h(\tau)\big)_{xv}\le-1+\frac92h,\\
 0<\big(M_h(\tau)\big)_{vv}\le3+\frac{21}2h,\qquad
 \frac{21}8\le\operatorname{tr}M_h(\tau)\le\frac{15}4 .
 \end{gathered}
 \label{eq:ud-M-entries}
\end{equation}
Here $(M_h)_{vv}>0$ because $M_h\succ0$, and $\operatorname{tr}M_h\le\max\{\operatorname{tr}\Theta_\tau,\operatorname{tr}\Theta_{\tau-h}\}\le\frac{15}4$.

\emph{Bounds on $M_h^2$.} For a symmetric $2\times2$ matrix $X$,
\[
    (X^2)_{xx}=X_{xx}^2+X_{xv}^2,\qquad
    (X^2)_{xv}=X_{xv}\operatorname{tr}X,\qquad
    (X^2)_{vv}=X_{xv}^2+X_{vv}^2 .
\]
With \eqref{eq:ud-M-entries} and $h\le\frac1{64}$, the last two give
\begin{equation}
 \begin{gathered}
 \frac{21}8\Big(1-\frac92h\Big)\le-\big(M_h(\tau)^2\big)_{xv}\le\frac{43}{40}\cdot\frac{15}4,\\
 \big(M_h(\tau)^2\big)_{vv}\le\Big(\frac{43}{40}\Big)^2+\Big(3+\frac{21}{128}\Big)^2\le\frac{56}5 .
 \end{gathered}
 \label{eq:ud-M-square}
\end{equation}
For the position entry we use the convexity of $X\mapsto X^2$,
\[
    M_h(\tau)^2\preceq(1-u)W^2+u\,\Theta_{\tau-h}^2
    \qquad\big(\text{the difference is }u(1-u)(W-\Theta_{\tau-h})^2\big),
\]
together with $(\Theta_{\tau-h}^2)_{xx}\le\frac54$ and, by \eqref{eq:ud-W-lower} (note $|W_{xv}|\le1+\frac{24}5\tau$),
\[
    (W^2)_{xx}\le\Big(\frac12+\frac{11}5\tau\Big)^2+\Big(1+\frac{24}5\tau\Big)^2
    =\frac54+\frac{59}5\tau+\frac{697}{25}\tau^2
    \le\frac54+\frac{25}2\tau ,
\]
where $\frac{697}{25}\tau^2\le\frac{697}{1600}\tau$. Hence
\begin{equation}
 \big(M_h(\tau)^2\big)_{xx}\le(1-u)\Big(\frac54+\frac{25}2\tau\Big)+\frac54u\le\frac54+\frac{25}2\tau .
 \label{eq:ud-M-position}
\end{equation}

\paragraph{Step 5: the three entries.}
We insert the bounds of Steps~2--4 into \eqref{eq:ud-certificate-matrix}. For $\bar A_\tau$ we use \eqref{eq:ud-scalar-A} and the averages
\[
    \frac1h\int_{\tau-h}^{\tau}t\dd t=\tau-\frac h2,
    \qquad
    \frac1h\int_{\tau-h}^{\tau}t^+\dd t=\frac{\tau^2}{2h}\le\frac h2 .
\]

\emph{Position entry.} By \eqref{eq:ud-scalar-A}, \eqref{eq:ud-scalar-plus}, \eqref{eq:ud-rank-bounds} and \eqref{eq:ud-M-position},
\begin{align*}
 \big(\mathcal D_h(\tau)\big)_{xx}
 &\ge 2+9\Big(\tau-\frac h2\Big)-\frac13\Big(\frac54+8\tau\Big)-\frac h5
       -\frac14\Big(\frac54+\frac{25}2\tau\Big)\\
 &=\frac{61}{48}+\frac{77}{24}\tau-\frac{47}{10}h
 \;\ge\;\frac{61}{48}-\frac{47}{640}
 =\frac{2299}{1920}
 >\frac{19}{16} .
\end{align*}

\emph{Velocity entry.} By \eqref{eq:ud-scalar-A}, \eqref{eq:ud-scalar-plus}, \eqref{eq:ud-rank-bounds} and \eqref{eq:ud-M-square},
\begin{align*}
 \big(\mathcal D_h(\tau)\big)_{vv}
 &\ge10+31\Big(\tau-\frac h2\Big)-\frac13(10+35\tau)-\Big(48\Lambda^2+\frac{16}5\Big)
       -\frac14\cdot\frac{56}5+112\Lambda^2\\
 &=64\Lambda^2+\frac23+\frac{58}3\tau-\frac{31}2h
 \;\ge\;64\Lambda^2+\frac23-\frac{31}{128}
 >64\Lambda^2 .
\end{align*}

\emph{Off-diagonal entry.} From below, by \eqref{eq:ud-scalar-A}, \eqref{eq:ud-scalar-plus}, \eqref{eq:ud-rank-bounds} and \eqref{eq:ud-M-square},
\begin{align*}
 \big(\mathcal D_h(\tau)\big)_{xv}
 &\ge\Big(-\frac92-9h\Big)+\frac76-\frac25+\frac14\cdot\frac{21}8\Big(1-\frac92h\Big)\\
 &=-\frac{1477}{480}-\frac{765}{64}h
 \;\ge\;-\frac{1477}{480}-\frac3{16}
 =-\frac{1567}{480}
 >-\frac{17}5 .
\end{align*}
From above, $(\bar A_\tau)_{xv}<0$ by \eqref{eq:ud-scalar-A}, so the same bounds give
\[
 \big(\mathcal D_h(\tau)\big)_{xv}
 \le\frac{19}{15}+\frac25+\frac14\cdot\frac{43}{40}\cdot\frac{15}4
 =\frac53+\frac{129}{128}
 <\frac{17}5 .
\]
This proves \eqref{eq:ud-certificate-entries}, and hence the lemma.
\end{proof}

\section{Elementary critical-damping formulas}
\label{app:critical}

The estimates of Section~\ref{sec:underdamped-stat} and Appendix~\ref{app:ud-proofs} use only the following explicit identities for the matrix $C$ of \eqref{eq:matrices}. Writing $C=\Itwo+K_0$ with $K_0^2=0$ (critical damping makes $C-\Itwo$ nilpotent),
\begin{equation}
    e^{tC}=e^t\begin{pmatrix}(1-t)\Id&-t\Id\\t\Id&(1+t)\Id\end{pmatrix},
    \qquad
    e^{-tC}=e^{-t}\begin{pmatrix}(1+t)\Id&t\Id\\-t\Id&(1-t)\Id\end{pmatrix}.
    \label{eq:critical-exponentials}
\end{equation}
The singular values of $C$ are $1+\sqrt2$ and $\sqrt2-1$, each repeated $d$ times. Finally,
\begin{equation}
    e^{tC}e^{tC^\top}
    =e^{2t}\begin{pmatrix}(1-2t+2t^2)\Id&-2t^2\Id\\-2t^2\Id&(1+2t+2t^2)\Id\end{pmatrix}.
    \label{eq:critical-gramian}
\end{equation}
These formulas verify all block identities used in Section~\ref{sec:ud-integrator2} and in Appendix~\ref{app:ud-proofs} and show explicitly how velocity expansion is transferred to position through the critically damped coupling.

% Appendix section `Numerical experiments: details' with all number macros replaced by literal values
% (snapshot of the runs of 2026-08-31; regenerate via the macro version if experiments are rerun).

\section{Numerical experiments: details}
\label{app:experiments}

This appendix documents the experiments of Section~\ref{sec:experiments}. Both use scores that are available in closed form, so that initialization and discretization are the only error sources, exactly as in the theory; nothing is trained until Section~\ref{app:experiments-ddpm}.

\subsection{The initialization error evolves by the forward semigroup}
\label{app:experiments-semigroup}

The one-dimensional experiment computes the law of the exact reverse flow started from the wrong initialization \emph{without simulating it}, through the following consequence of the transport equation \eqref{eq:h-transport}. Let $P^{\mathrm{fwd}}_t$ denote the Markov semigroup of the generic forward diffusion \eqref{eq:generic-forward}, $P^{\mathrm{fwd}}_tf(z)=\E[f(Z_t)\mid Z_0=z]$.

\begin{lemma}[Semigroup representation of the initialization error]
\label{lem:semigroup-representation}
In the setting of Lemma~\ref{lem:generic-current}, $p_t=q_t\,P^{\mathrm{fwd}}_t(p_0/q_0)$ for all $t\in[0,T]$. In particular, for the OU process \eqref{eq:forward-ou}, $p_t(x)=q_t(x)\,\E\big[(p_0/q_0)\big(e^{-t}x+\sqrt{1-e^{-2t}}\,Z\big)\big]$ with $Z\sim N(0,\Id)$.
\end{lemma}

\begin{proof}
By \eqref{eq:h-transport}, $h_t=p_t/q_t$ satisfies $\partial_th_t=-v_t^q\cdot\nabla h_t+L_th_t$. Inserting $v_t^q=-b+G\nabla\log q_t$ and $L_th=G:\nabla^2h+G\nabla\log q_t\cdot\nabla h$, the two terms containing $\nabla\log q_t$ cancel, and
\[
    \partial_th_t=b\cdot\nabla h_t+G:\nabla^2h_t ,
\]
the backward Kolmogorov equation of \eqref{eq:generic-forward}, whose solution with initial datum $h_0$ is $P^{\mathrm{fwd}}_th_0$. (Probabilistically: under the exact reverse path measure, $\frac{\dd p_t}{\dd q_t}(y)=\E[h_0(Y_0)\mid Y_t=y]$, which by time reversal and the Markov property of the forward process equals $\E[h_0(Z_T)\mid Z_{T-t}=y]=P^{\mathrm{fwd}}_th_0(y)$.) For the OU process, $P^{\mathrm{fwd}}_tf(x)=\E f\big(e^{-t}x+\sqrt{1-e^{-2t}}\,Z\big)$.
\end{proof}

\subsection{One-dimensional Gaussian mixture}
\label{app:experiments-gmm}

\paragraph{Setting.}
The data law is $\rho_0=\sum_{i=1}^3w_iN(\mu_i,\sigma^2)$ with $w=(0.3,0.4,0.3)$, $\mu=(-20,0,20)$, $\sigma=1.5$; the forward process is the OU diffusion \eqref{eq:forward-ou} on $[0,T]$ with $T=5$, so $q_t=\rho_{T-t}=\sum_iw_iN(e^{-s}\mu_i,e^{-2s}\sigma^2+1-e^{-2s})$, $s=T-t$, and $s_t=\nabla\log q_t$ are explicit. The uniform score Lipschitz constant of Assumption~\ref{ass:lipschitz} is $L_s=\sup_{t,x}|s_t'(x)|\approx20$. The wrong initialization is $p_0=N(10,0.5^2)$; the exact initialization $q_0=\rho_T$ is a near-standard Gaussian, so both $H_0=\KL(p_0\,\|\,q_0)=49.8$ and $\F_0=100$ are large.

\paragraph{Exact reverse overdamped flow.}
By Lemma~\ref{lem:semigroup-representation}, $p_t=q_tP^{\mathrm{fwd}}_th_0$ with $h_0=p_0/q_0$, and the relative score has the posterior-mean form $\nabla\log P^{\mathrm{fwd}}_th_0(x)=\frac{e^{-t}}{1-e^{-2t}}\E_{\nu_x}[Y-e^{-t}x]$ with $\nu_x(y)\propto N(y;e^{-t}x,1-e^{-2t})h_0(y)$, which is evaluated using numerical quadrature algorithms. The mass of $p_T$ in the three basins (split at the midpoints between the modes) is $(0.05, 0.25, 0.71)$, against $(0.3, 0.4, 0.3)$ for $\rho_0$.

\paragraph{Discretized sampler.}
The frozen-score exponential integrator \eqref{eq:explicit-step} is run on uniform grids with $N=T/h\in\{10,25,50,100,250,500,1000\}$ steps. Its law $\hp_t$ is propagated on a spatial grid by applying the Gaussian transition kernel of \eqref{eq:explicit-step} exactly (a dense kernel--vector product), including at intermediate times $t\in(t_k,t_{k+1})$, and $\frac1T\int_0^T\F(\hp_t\,\|\,q_t)\dd t$ is computed by four-point Gauss--Legendre quadrature in time within every step. No trajectories are sampled. The gap of the time-averaged Fisher divergence
decays like $O(h)$ (Figure~\ref{fig:gmm1d}(c)), matching the $O(h)$ discretization term of Theorem~\ref{thm:overdamped-stationarity}.

\paragraph{Critically damped flow.}
Figure~\ref{fig:gmm1d-app} repeats the whole experiment for the kinetic model of Section~\ref{sec:underdamped-stat} ($\gamma=2$, $\xi=1$), with data law $\rho_0\otimes N(0,1)$ on phase space and the same wrong position law, $p_0=N(10,0.5^2)\otimes N(0,1)$. Lemma~\ref{lem:semigroup-representation} applies verbatim with the kinetic forward semigroup, whose transition law is Gaussian with mean $e^{-tC}z$ and covariance $G_t$ (Section~\ref{sec:ud-integrator2}). In panel~(a), the modified Fisher divergence $\F_P$ with $P=P_2$ decays faster than the guaranteed $e^{-t}$. The terminal TV is larger than in the overdamped case because the critically damped forward process contracts positions only like $(1+s)e^{-s}$, so at $T=5$ it has mixed less and the mode weights are decided earlier in reverse time: panel~(b) shows the terminal position marginals---$p_T$ puts mass $(0.00, 0.00, 1.00)$ on the three basins, against $(0.3, 0.4, 0.3)$ for $\rho_0$, i.e.\ almost everything on the mode nearest the initialization---while the marginal relative score remains small, the same locally-consistent, globally-reweighted picture as in Figure~\ref{fig:gmm1d}(b).

\paragraph{Kinetic exponential integrator.}
We also ran the frozen-score exponential integrator \eqref{eq:update} itself, in the two-factor form \eqref{eq:factorization-main}, on uniform grids with $h\in\{0.1,0.05,0.025,0.0125,0.00625\}$ and $t_N=T$: the sampler $\hp$ (wrong initialization) and its twin $\hq$ (exact initialization $\rho_T$) of Theorem~\ref{thm:underdamped-stationarity} are propagated through the same kernels on the $561\times281$ phase-space grid, and $\int_0^h\F(\hp_{i,\tau}\,\|\,\hq_{i,\tau})\dd\tau$ is accumulated by four-point Gauss--Legendre quadrature along the within-step interpolation \eqref{eq:interpolation}, exactly the left-hand sides of the theorem. The time average of $\F(\hp\,\|\,\hq)$ carries almost no discretization term---the divergence $\F(\hp\,\|\,\hq)$ measures initialization error alone, as the decomposition of Section~\ref{sec:underdamped-stat} intends.

Although we do not prove this directly, the sampler is indeed empirically first-order stationary with respect to the exact marginals themselves. For the comparison with Figure~\ref{fig:gmm1d}(c) we evaluate $\frac1T\int_0^T\F(\hp_t\,\|\,q_t)\dd t$ against the analytic scores along the frozen-score SDE interpolation of the scheme, the kinetic analog of \eqref{eq:numerical-sde}. Its gap to the exact-flow value decays like $O(h)$ (Figure~\ref{fig:gmm1d-app}(d)).

\begin{figure}[t]
\centering
\includegraphics[width=0.72\textwidth]{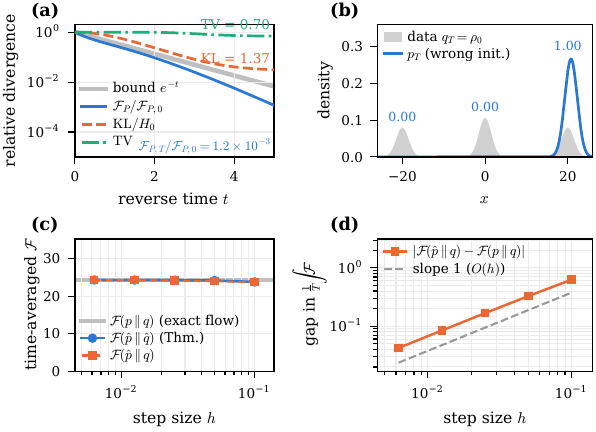}
\caption{The critically damped experiment ($\gamma=2$, $\xi=1$), presented as in Figure~\ref{fig:gmm1d}. \textbf{(a)}~Along the exact reverse flow, the modified Fisher divergence $\F_P$ contracts faster than the guaranteed $e^{-t}$, $\KL$ decreases monotonically, and $\mathrm{TV}$ plateaus. \textbf{(b)}~Terminal position marginal. \textbf{(c)}~The frozen-score exponential integrator \eqref{eq:update}: the time-averaged $\F(\hp\,\|\,\hq)$ of Theorem~\ref{thm:underdamped-stationarity} (circles) and the time-averaged $\F(\hp_t\,\|\,q_t)$ against the exact marginals along the frozen-score SDE interpolation (squares) versus step size, with the exact-flow value $\frac1T\int_0^T\F(p_t\,\|\,q_t)\dd t$ (gray). \textbf{(d)}~Gap between the time-averaged $\F(\hp_t\,\|\,q_t)$ of (c) and its exact-flow counterpart $\F(p_t\,\|\,q_t)$ versus step size (log--log; dashed: slope one), as in Figure~\ref{fig:gmm1d}(c).}
\label{fig:gmm1d-app}
\end{figure}

\subsection{CelebA-HQ with the exact empirical score}
\label{app:experiments-celeba}

\paragraph{Setting.}
Let $x_1,\dots,x_n\in[-1,1]^d$ be the $n=28000$ training images of CelebA-HQ \citep{karras2018progressive} at $256\times256$ resolution, and $\rho_0=\frac1n\sum_i\delta_{x_i}$. The OU marginals are $\rho_s=\frac1n\sum_iN(a_sx_i,b_s^2\Id)$ with $a_s=e^{-s}$, $b_s^2=1-e^{-2s}$---a Gaussian kernel density estimate whose bandwidth is the noise level---and
\[
    \nabla\log\rho_s(y)=\frac{a_s\bar x_s(y)-y}{b_s^2},
    \ \ 
    \bar x_s(y)=\sum_iw_i^s(y)\,x_i,
    \ \ 
    w^s(y)=\operatorname{softmax}_i\Big(\frac{2a_s\langle y,x_i\rangle-a_s^2\|x_i\|^2}{2b_s^2}\Big)\!
\]
(the common term $\|y\|^2$ is dropped from the softmax), where $\bar x_s(y)=\E[X_0\mid X_s=y]$ is the posterior mean. One score evaluation is one matrix product with the $n\times d$ data matrix. Since the data are exactly memorized, every trajectory ends at a training image, and what the initialization can change is \emph{which} training images are produced.

That the samples are training images is not a shortcoming of the experiment but the setting of the paper, which assumes the true score throughout: for a finite training set the true score is the empirical one above, and it is also the exact minimizer of the denoising score-matching objective, so a diffusion model whose network fitted its objective perfectly would memorize; the generalization of trained models comes from the inductive bias of the network, not from the objective \citep{yoon2023diffusion,gu2023memorization,kadkhodaie2024generalization,li2024good}. Using the exact score is therefore the only way to run the object analyzed in Sections~\ref{sec:decay}--\ref{sec:overdamped-stat} on real data with the score error removed, leaving initialization and discretization as the only error sources. The phenomenon shown here is moreover not tied to memorization: Appendix~\ref{app:experiments-ddpm} shows the same phenomenon with a learned score that does not memorize.

\paragraph{Sampler.}
We use the frozen-score exponential integrator \eqref{eq:explicit-step} on a grid of forward times $s_0=T>s_1>\dots>s_{K-1}=s_{\min}>s_K=0$ that is geometric in the noise-to-signal ratio $\sigma_s=\sqrt{e^{2s}-1}$ \citep{karras2022elucidating} between $s_0=T$ and $s_{K-1}=s_{\min}$, followed by one last step to $s=0$ ($K=150$, $s_{\min}=1\times 10^{-5}$, $T=5.06$). The horizon is $T=5.06$, the horizon of the DDPM noise schedule used in Appendix~\ref{app:experiments-ddpm}, so that the two experiments run the same three initial laws over the same time span; at this $T$ the standard initialization is close to exact ($\KL(N(0,\Id)\,\|\,\rho_T)=0.30$ nats) and serves as the control. Each initialization is run for $64$ samples.

\paragraph{Initializations.}
(i)~\emph{standard}: $\hp_0=N(0,\Id)$, the initialization of every practical sampler; (ii)~\emph{mean shift}: $\hp_0=N(\mu,\Id)$ with $\mu=\beta\,\mathbf 1/\sqrt d$, a uniform brightness offset of $\beta/\sqrt d$ per pixel, $\beta=\pm4$---the same three initial laws as in the learned-score experiment of Appendix~\ref{app:experiments-ddpm}. $H_0$ and $\F_0$ are estimated by Monte Carlo from the closed-form densities. 

\paragraph{Results.}
Figure~\ref{fig:celeba} shows the $64$ generated images of each run. Every one of them is close to a training image (root-mean-square distance to the nearest training image below one gray level, i.e.\ the noise of the last step): local consistency holds regardless of the initialization. Which training images are produced does depend on the initialization. The standard initialization $N(0,\Id)$, nearly exact at this horizon ($H_0=0.30$ nats), is the control: it returns a mixed draw of the training set---dark and light backgrounds, hair colors and lighting conditions (mean pixel value $-0.14$ on the $[-1,1]$ scale, mean per-image pixel standard deviation $0.51$). The mean shifts ($H_0=5.03$ resp.\ $6.78$ nats) move the output to opposite ends of the training set: $\beta=-4$ returns almost exclusively dark photographs---dim lighting, dark backgrounds, hair and skin (mean pixel value $-0.43$)---and $\beta=+4$ almost exclusively bright ones---light backgrounds, blond hair, pale skin, over-exposed shots. The initialization error is not negligible but shrinks along the flow in the first-order sense, yet it decides the palette of the output: small relative Fisher divergence certifies local consistency, not the mode weights. 

\begin{figure}[t]
\centering
\begin{minipage}[t]{0.32\textwidth}\centering
\includegraphics[width=\linewidth]{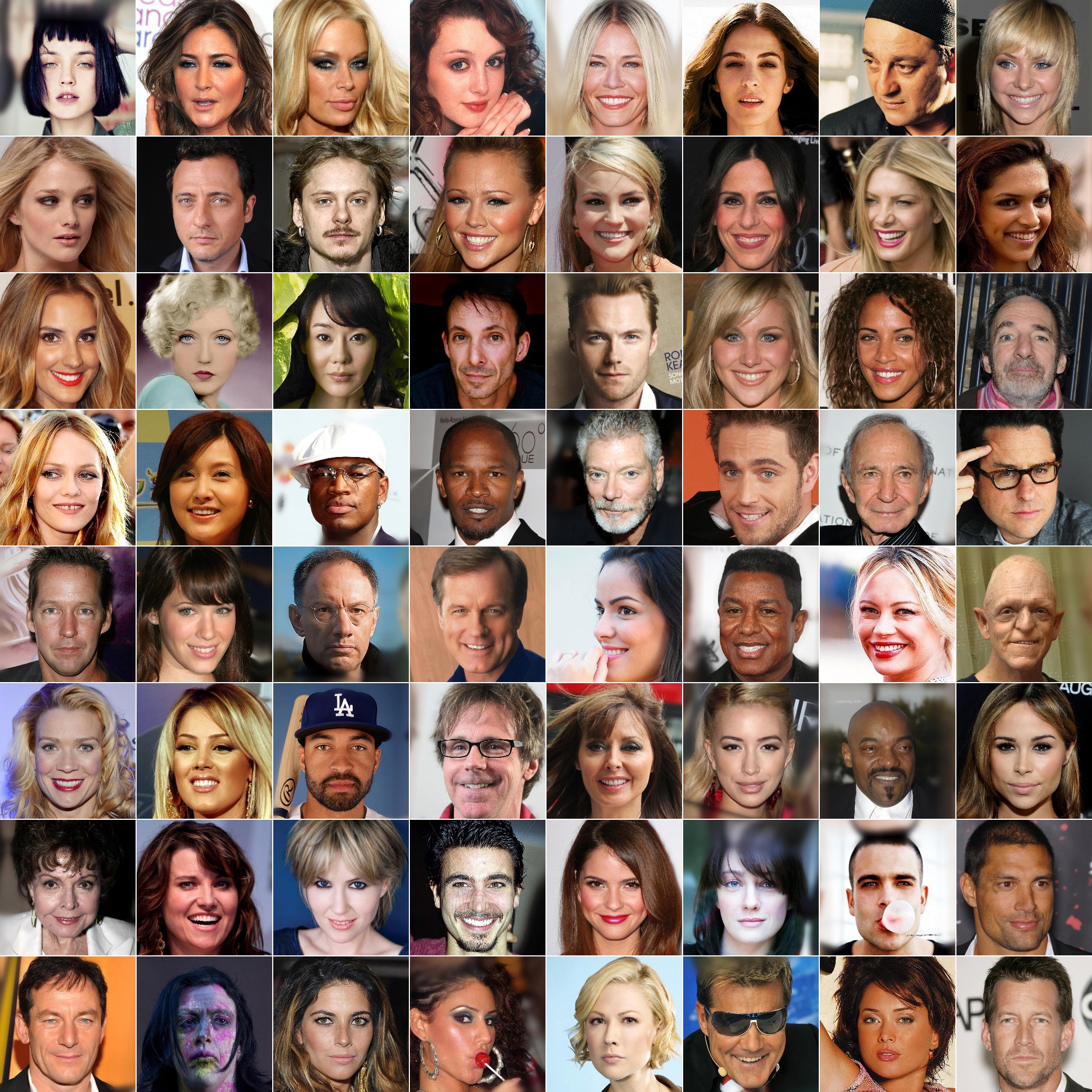}\\[1pt]{\small $N(0,\Id)$ }
\end{minipage}\hfill
\begin{minipage}[t]{0.32\textwidth}\centering
\includegraphics[width=\linewidth]{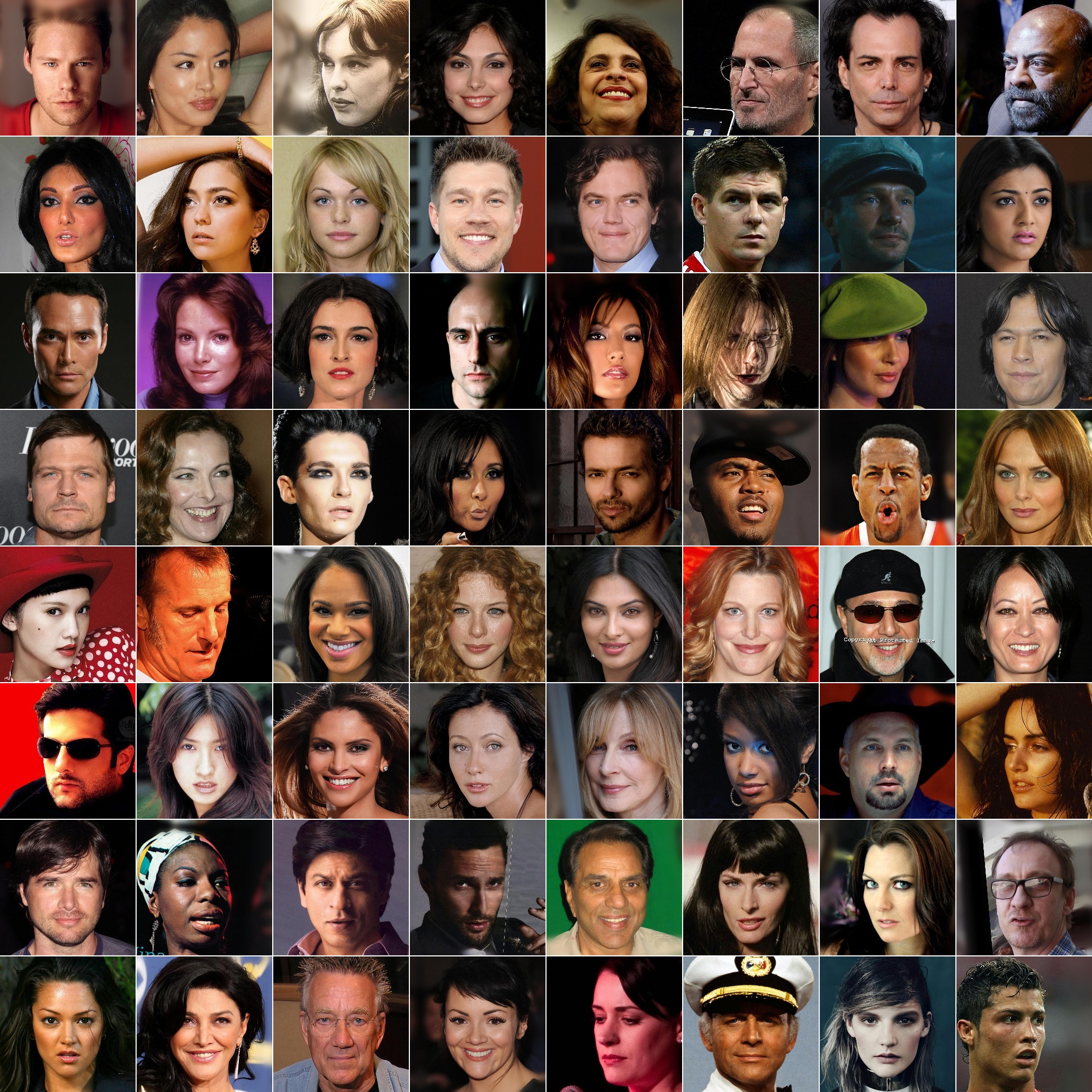}\\[1pt]{\small $N(-4\,\mathbf 1/\sqrt d,\Id)$ }
\end{minipage}\hfill
\begin{minipage}[t]{0.32\textwidth}\centering
\includegraphics[width=\linewidth]{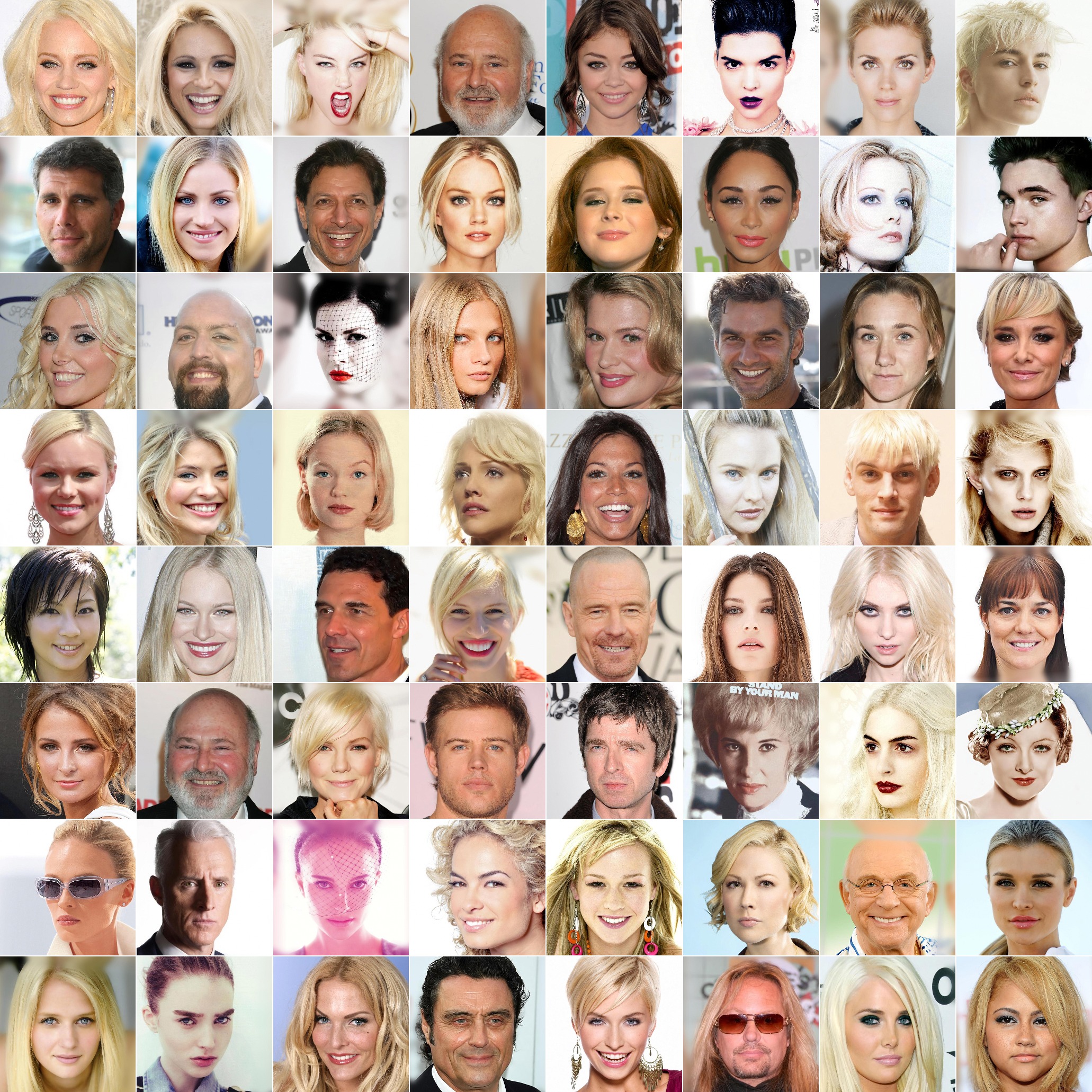}\\[1pt]{\small $N(4\,\mathbf 1/\sqrt d,\Id)$ }
\end{minipage}
\caption{CelebA-HQ with the exact empirical score: $64$ images generated by the sampler of Section~\ref{sec:overdamped-stat} from the standard initialization and the two mean shifts, $T=5.06$, $n=28000$ training images. Every image is a training image; the initialization decides which ones are produced.}
\label{fig:celeba}
\end{figure}

\subsection{The same effect with a learned score}
\label{app:experiments-ddpm}

To check that this phenomenon persists in the presence of score estimation errors, we repeat the mean-shift experiment with the pretrained CelebA-HQ $256\times256$ DDPM of \citet{ho2020denoising} (\texttt{google/ddpm-celebahq-256}), sampled with its ancestral sampler ($250$ steps)---the discrete predecessor of the integrator of Section~\ref{sec:overdamped-stat}---whose noise schedule corresponds to the OU horizon $T=5.06$. Figure~\ref{fig:ddpm} shows $64$ samples from the standard initialization $N(0,\Id)$ and from $N(\mu,\Id)$ with $\mu=\beta\,\mathbf 1/\sqrt d$ for $\beta=-4$ and $+4$---the same three initial laws as in Figure~\ref{fig:celeba}. The mean pixel value of the samples moves from $-0.09$ to $-0.56$ and $+0.50$ respectively: a positive shift makes the learned model return bright photographs and a negative shift dark ones. The bias now acts on the cluster of bright (dark) photographs rather than on individual training images, as one expects from a score that has learned the data distribution rather than the training set; the mechanism---the mass allocated to the clusters is decided by the tilt the initialization induces, while every sample remains a clean face---is the same. 

\begin{figure}[t]
\centering
\begin{minipage}[t]{0.32\textwidth}\centering
\includegraphics[width=\linewidth]{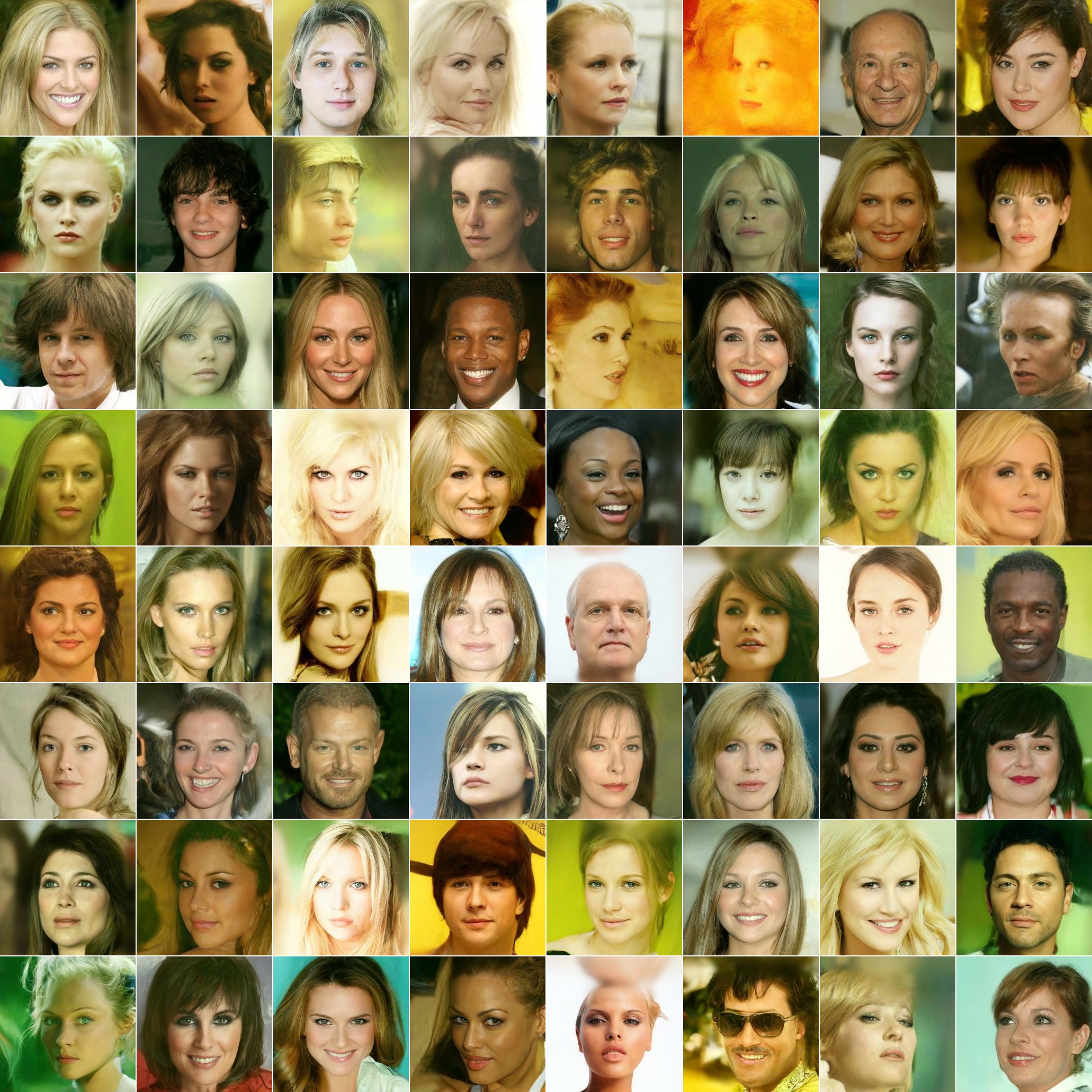}\\[1pt]{\small $N(0,I)$\ (mean pixel $-0.09$)}
\end{minipage}\hfill
\begin{minipage}[t]{0.32\textwidth}\centering
\includegraphics[width=\linewidth]{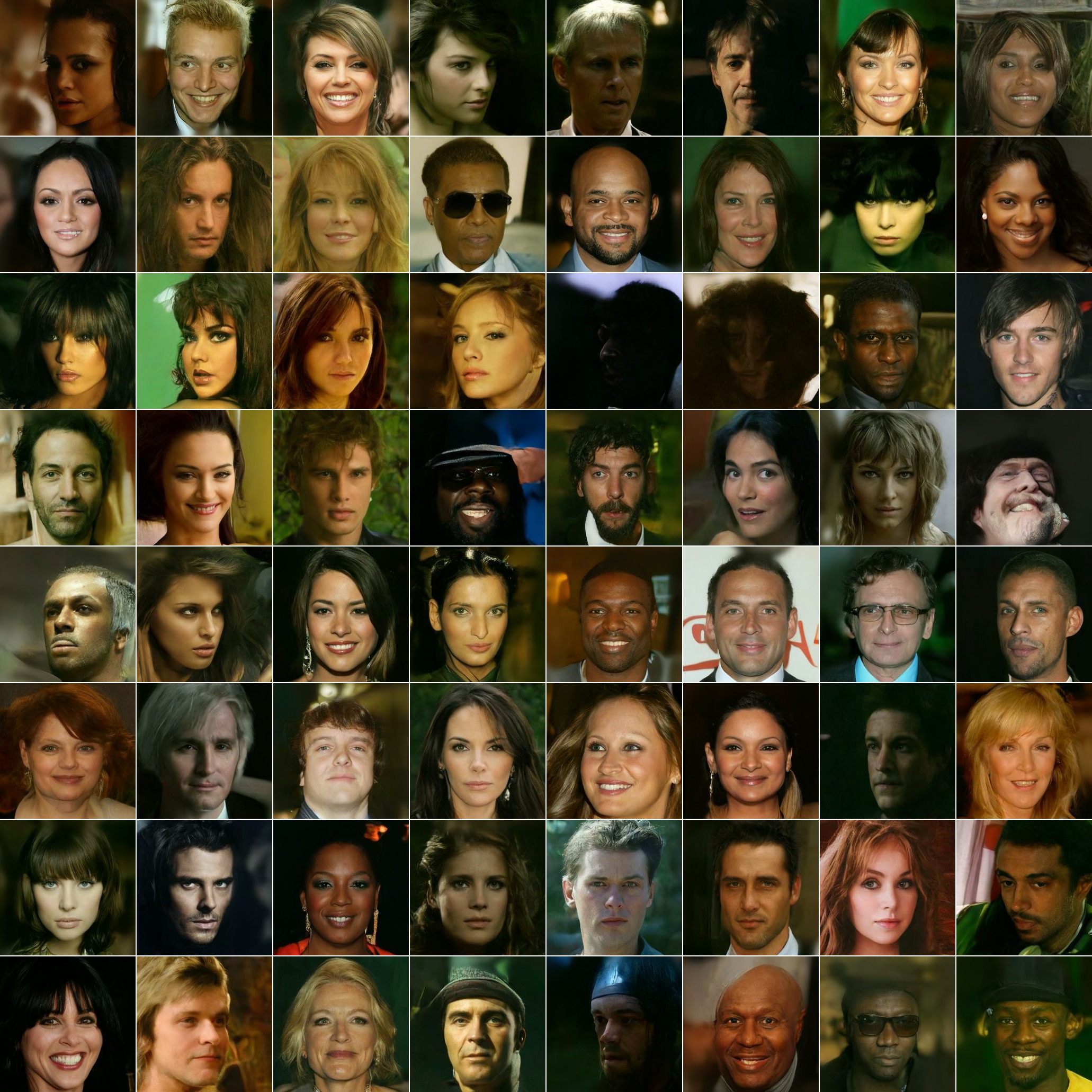}\\[1pt]{\small $N(\mu,I)$, $\mu=-4\,\mathbf{1}/\sqrt d$\ ($-0.56$)}
\end{minipage}\hfill
\begin{minipage}[t]{0.32\textwidth}\centering
\includegraphics[width=\linewidth]{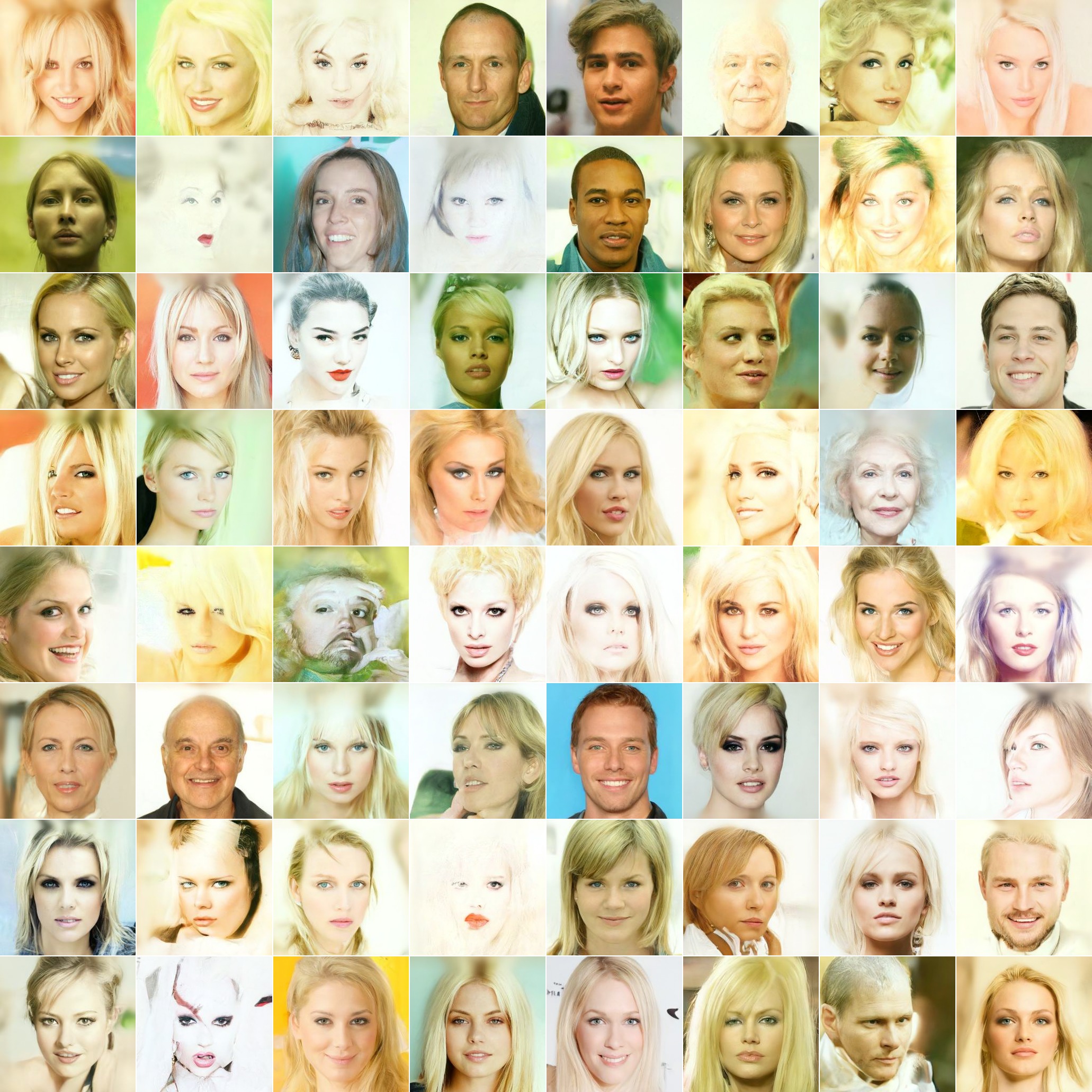}\\[1pt]{\small $N(\mu,I)$, $\mu=4\,\mathbf{1}/\sqrt d$\ ($+0.50$)}
\end{minipage}
\caption{Pretrained CelebA-HQ DDPM \citep{ho2020denoising}, $64$ samples of the $250$-step ancestral sampler from the same three initializations of the initial noise as in Figure~\ref{fig:celeba}. The mean shifts along the all-ones direction nevertheless bias them towards dark ($\beta<0$) or bright ($\beta>0$) photographs, exactly as with the exact score.}
\label{fig:ddpm}
\end{figure}

\end{document}